\documentclass{article} % For LaTeX2e
\usepackage[utf8]{inputenc}
\usepackage[T1]{fontenc}
\usepackage{iclr2025_conference,times}

\usepackage{amsmath,amsfonts,bm}

\def\eqref#1{equation~\ref{#1}}
\def\1{\bm{1}}

\DeclareMathAlphabet{\mathsfit}{\encodingdefault}{\sfdefault}{m}{sl}
\SetMathAlphabet{\mathsfit}{bold}{\encodingdefault}{\sfdefault}{bx}{n}

\usepackage{hyperref}
\usepackage{url}
\usepackage{algorithm}    
\usepackage{algorithmic}  
\usepackage{graphicx}
\usepackage{amsmath}  % 数学宏包，提供行间公式、矩阵等
\usepackage{amsthm}   % 定理环境定义
\usepackage{amsfonts} % 提供\mathbb等字体
\usepackage{amssymb}  % 提供\lesssim等扩展数学符号
\usepackage{mathtools}
\usepackage{booktabs}
\usepackage{wrapfig}
\usepackage{cleveref}
\usepackage{subcaption}
\usepackage{listings}
\usepackage{xcolor}
\usepackage{courier}
\usepackage{appendix}

\definecolor{codegreen}{rgb}{0,0.6,0}
\definecolor{codegray}{rgb}{0.5,0.5,0.5}
\definecolor{codepurple}{rgb}{0.58,0,0.82}
\definecolor{backcolour}{rgb}{0.97,0.97,0.97}
\definecolor{framecolor}{rgb}{0.8,0.8,0.8}

\lstdefinestyle{pythonstyle}{
	backgroundcolor=\color{backcolour},   
	commentstyle=\color{codegreen}\itshape,
	keywordstyle=\color{blue}\bfseries,
	numberstyle=\tiny\color{codegray},
	stringstyle=\color{codepurple},
	basicstyle=\ttfamily\scriptsize,
	breakatwhitespace=false,         
	breaklines=true,                 
	captionpos=b,                    
	keepspaces=true,                 
	numbers=left,                    
	numbersep=8pt,                  
	showspaces=false,                
	showstringspaces=false,
	showtabs=false,                  
	tabsize=4,
	frame=single,
	rulecolor=\color{framecolor},
	xleftmargin=1.5em,
	framexleftmargin=1.2em,
	columns=flexible,
	aboveskip=1.2em,
	belowskip=1.2em,
	lineskip=0.1em,
	escapeinside={(*@}{@*)},
}

\lstdefinelanguage{Python}{
	keywords={def, class, if, else, elif, for, while, return, import, from, as, try, except, finally, with, lambda, True, False, None, and, or, not, in, is, assert, break, continue, pass, global, nonlocal, yield},
	keywordstyle=\color{blue}\bfseries,
	sensitive=true,
	comment=[l]{\#},
	commentstyle=\color{codegreen}\itshape,
	string=[s]{"}{"},
	stringstyle=\color{codepurple},
	showstringspaces=false,
	morecomment=[s]{"""}{"""},
	morestring=[s]{'}{'},
}

\providecommand{\wp}{\ensuremath{\mathcal{P}}}

\newtheorem{theorem}{Theorem}[section]  
\newtheorem{definition}[theorem]{Definition}
\newtheorem{lemma}[theorem]{Lemma}

\title{Beyond Flattening: A Geometrically Principled Positional Encoding for Vision Transformers with Weierstrass Elliptic Functions}

\author{
	Zhihang Xin \\
	School of Science, Jiangnan University \\
	\texttt{1131230221@stu.jiangnan.edu.cn}
	\And
	Xitong Hu \\
	School of Science, Jiangnan University \\
	\texttt{1131230110@stu.jiangnan.edu.cn}
	\And
	Rui Wang\thanks{Corresponding author.} \\
	School of Artificial Intelligence and Computer Science \\
	Jiangnan University \\
	\texttt{cs\_wr@jiangnan.edu.cn}
}

\usepackage{fancyhdr}

\iclrfinalcopy % Uncomment for camera-ready version, but NOT for submission.

\begin{document}
	
	\maketitle

\begin{abstract}
Vision Transformers have demonstrated remarkable success in computer vision tasks, yet their reliance on learnable one-dimensional positional embeddings fundamentally disrupts the inherent two-dimensional spatial structure of images through patch flattening procedures. Traditional positional encoding approaches lack geometric constraints and fail to establish monotonic correspondence between Euclidean spatial distances and sequential index distances, thereby limiting the model's capacity to leverage spatial proximity priors effectively. We propose Weierstrass Elliptic Function Positional Encoding (WEF-PE), a mathematically principled approach that directly addresses two-dimensional coordinates through natural complex domain representation, where the doubly periodic properties of elliptic functions align remarkably with translational invariance patterns commonly observed in visual data. Our method exploits the non-linear geometric nature of elliptic functions to encode spatial distance relationships naturally, while the algebraic addition formula enables direct derivation of relative positional information between arbitrary patch pairs from their absolute encodings. Comprehensive experiments demonstrate that WEF-PE achieves superior performance across diverse scenarios, including 63.78\% accuracy on CIFAR-100 from-scratch training with ViT-Tiny architecture, 93.28\% on CIFAR-100 fine-tuning with ViT-Base, and consistent improvements on VTAB-1k benchmark tasks. Theoretical analysis confirms the distance-decay property through rigorous mathematical proof, while attention visualization reveals enhanced geometric inductive bias and more coherent semantic focus compared to conventional approaches.The source code implementing the methods described in this paper is publicly available on GitHub.
\end{abstract}

\section{INTRODUCTION}
\label{sec:introduction}

Vision Transformers (ViTs)(\cite{dosovitskiy2021image}) have emerged as a dominant architecture in computer vision, fundamentally challenging the long-standing prevalence of Convolutional Neural Networks (CNNs)(\cite{lecun1998gradient}). By treating an image as a sequence of patches, ViTs leverage the self-attention mechanism to capture global dependencies, a stark contrast to the localized inductive biases inherent in CNNs. However, this flexibility comes at a cost: ViTs lack an intrinsic understanding of spatial geometry. Consequently, their performance is critically dependent on an external mechanism to supply positional information---a role fulfilled by positional encodings (PE).

The standard approach in ViTs employs simple, learnable 1D positional embeddings. This design, however, suffers from a fundamental flaw: it necessitates the "flattening" of the 2D patch grid into a 1D sequence. This process irrevocably disrupts the image's intrinsic spatial structure. For instance, the sequential distance between vertically adjacent patches becomes artificially inflated compared to horizontally adjacent ones. More critically, these embeddings function as a mere lookup table, devoid of any geometric constraints. There is no guaranteed monotonic relationship between the true Euclidean distance of two patches in the image and the distance between their representations in the embedding space. This deficiency severely limits the model's ability to leverage spatial proximity priors, a cornerstone of effective visual understanding.

To address these fundamental limitations, we introduce the Weierstrass Elliptic Function Positional Encoding (WEF-PE), a mathematically principled framework grounded in the rich theory of complex analysis. Instead of flattening the image, we map the 2D coordinates of each patch directly to the complex plane, preserving their geometric integrity. We then utilize the Weierstrass elliptic function, $\wp(z)$, a doubly periodic meromorphic function, to generate a continuous and structured representation of space. This approach is not an arbitrary choice; the core properties of the elliptic function are remarkably well-suited for visual data. Its double periodicity naturally aligns with the translational patterns found in images, while its continuity provides inherent resolution invariance---a critical advantage for fine-tuning at different image sizes. Furthermore, we demonstrate that our method possesses a theoretically guaranteed distance-decay property and that the function's algebraic addition formula allows for the direct derivation of relative positional information.

This paper presents the comprehensive development, theoretical grounding, and empirical validation of WEF-PE. By replacing a heuristic design choice with a principled mathematical construct, we endow the Vision Transformer with a robust geometric inductive bias. Our main contributions and innovations are:
\begin{enumerate}
	\item We propose a novel positional encoding framework, WEF-PE, based on the Weierstrass elliptic function. It preserves the 2D spatial structure of images by mapping coordinates to the complex plane and provides a continuous, doubly periodic representation that is inherently resolution-invariant.
	\item We provide a rigorous theoretical analysis of our method, including a formal proof of its distance-decay property, which guarantees that spatial proximity is faithfully represented. We also show how the function's algebraic addition formula endows the model with direct relative position awareness.
	\item We conduct extensive experiments, demonstrating that WEF-PE significantly outperforms conventional methods in both from-scratch training and challenging fine-tuning scenarios on benchmarks like CIFAR-100 and VTAB-1k. Our qualitative analyses further reveal that WEF-PE provides a superior geometric inductive bias.
\end{enumerate}

\section{Preliminaries}
\label{gen_inst}

The Weierstrass elliptic function(\cite{weierstrass1854zur}), denoted by $\wp(z)$, is a fundamental example of a doubly periodic meromorphic function in complex analysis that establishes a connection between a complex torus (defined by a lattice) and an algebraic elliptic curve through a specific differential equation. The Weierstrass elliptic function possesses several key mathematical properties, such as being a doubly periodic meromorphic function (as detailed in \cref{def:elliptic_function}), satisfying a specific differential equation (\cref{thm:weierstrass_ode}), and adhering to a unique addition formula (\cref{thm:addition_formula}). These inherent mathematical properties render it a natural vehicle for processing two-dimensional images, thereby overcoming the limitations of conventional one-dimensional positional encodings. We provide a more comprehensive presentation of preliminaries and proofs on this topic in the Appendix(\cref{sec:supplementary_background}).

\begin{figure}[h!]
	\centering % 将整个图文块居中
	
	% 左侧的列，用于放置图片
	\begin{minipage}{0.5\textwidth}
		\centering
		\includegraphics[width=0.95\linewidth]{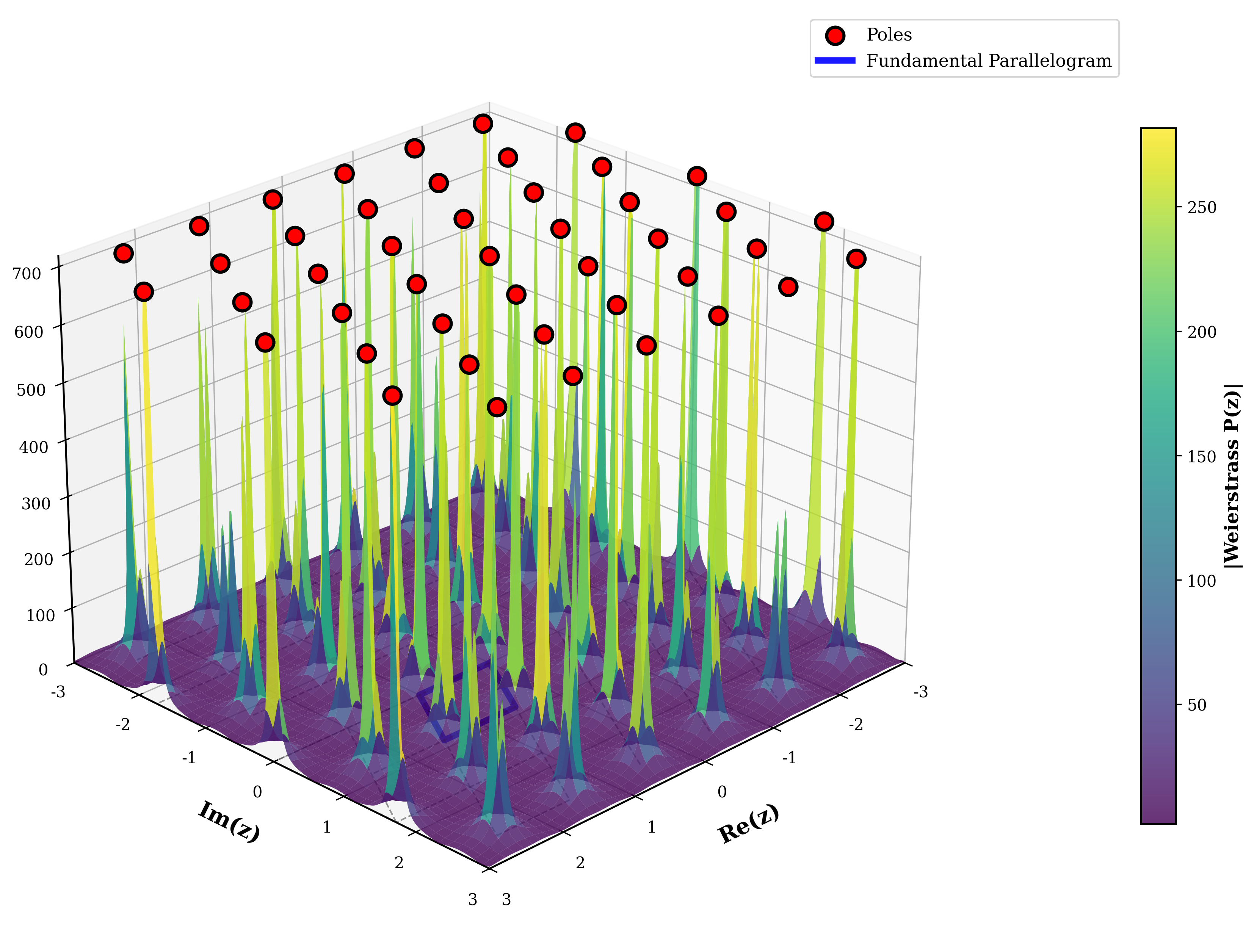}
	\end{minipage}%
	\hfill % 在两个列之间添加一些弹性空白
	% 右侧的列，用于放置图注
	\begin{minipage}{0.45\textwidth}
		\caption{Three-dimensional visualization of the Weierstrass elliptic function $\wp(z)$ demonstrating the doubly periodic structure and pole distribution across the complex plane, with fundamental parallelogram boundaries marked by red dots indicating lattice points.}
		\label{fig:elliptic_function}
	\end{minipage}
	
\end{figure}

\section{Methodology: Weierstrass Elliptic Function Positional Encoding}
\label{headings}

\subsection{Foundational Framework of WEF-PE}
\subsubsection{Principle of Doubly-Periodic Spatial Representation}

Traditional position encoding methods primarily rely on trigonometric functions or learnable parameters, making it challenging to fully express complex spatial relationships in images. The Weierstrass elliptic function, as a doubly periodic function, naturally possesses the capability to handle 2D spatial structures, while its nonlinear characteristics enable the capture of richer positional information.

\subsubsection{Coordinate System and Complex Plane Mapping}

Input images have varying resolutions, which leads to a variable number of image patches. Consequently, the direct use of absolute indices is problematic, as their valid range changes with the total number of patches. In contrast, the relative concept of position remains consistent within a normalized coordinate system.

Given an input image, we conventionally partition it into a grid of $H \times W$ patches.For patch coordinates $(i,j)$ of the input image, where $i \in \{0,1,\ldots,H-1\}, j \in \{0,1,\ldots,W-1\}$, we first normalize them to the $[0,1]$ interval:
\begin{equation}
	u = \frac{j + 0.5}{W}, \qquad v = \frac{i + 0.5}{H} \label{eq:uv_norm}
\end{equation}

The effective ranges for $u$ and $v$ are 
$\left(\frac{0.5}{W}, \frac{W-0.5}{W}\right)$ 
and 
$\left(\frac{0.5}{H}, \frac{H-0.5}{H}\right)$, 
which are proper subsets of $[0, 1]$. This is advantageous in certain situations. For example, $z=0$ is a pole in the mapping, and we want to avoid having the center of any patch located at a pole.

Subsequently, the normalized coordinates are mapped to the complex plane:
\begin{equation}
	z = \alpha_u \cdot u \cdot 2\text{Re}(\omega_1) + i \cdot \alpha_v \cdot v \cdot 2\text{Im}(\omega_3)
	\label{eq:complex_mapping}
\end{equation}
where $\alpha_u, \alpha_v$ are adjustable scaling factors, and $\omega_1, \omega_3$ represent the real half-period and imaginary half-period of the elliptic function, respectively.This is a pivotal step, as it embeds the rich geometric and analytic properties of the Weierstrass elliptic function into the spatial location of each image patch. In the process of positional encoding, the elliptic function is like "weaving a fishing net" within the image, truly coupling the two-dimensional information of the image together. 

\subsubsection{Feature Extraction from the Elliptic Function and its Derivative}

For the mapped complex number $z$, we compute the Weierstrass elliptic function value $\wp(z)$ and its derivative $\wp'(z)$:
\begin{equation}
	\wp_{\text{val}} = \wp(z), \qquad \wp_{\text{deriv}} = \wp'(z) \label{eq:wp_and_deriv}
\end{equation}

To fully utilize the information in complex numbers, we extract the real and imaginary parts:
\begin{equation}
	f_1 = \text{Re}(\wp(z)), \quad f_2 = \text{Im}(\wp(z)), \quad f_3 = \text{Re}(\wp'(z)), \quad f_4 = \text{Im}(\wp'(z)) \label{eq:features}
\end{equation}

This yields a 4-dimensional position feature vector $\mathbf{f} = [f_1, f_2, f_3, f_4]^T$.

In certain regions, the Weierstrass elliptic function may yield exceptionally large values, which can compromise training stability, potentially leading to gradient explosion and computational failure.For handling large values at the poles, we initially envisioned leveraging the favorable properties of Jacobi's theta functions for truncation. Theta functions are entire functions, their series always converge, and they typically do so rapidly, allowing high precision to be achieved with only a few terms. However, experiments revealed that while this scheme was mathematically more advantageous, its empirical performance was suboptimal. The fundamental reason is that this computational approach was overly complex, leading to significant instability during the training process.Herein, we present two empirically validated and robust solutions, which are respectively more suitable for pre-training and fine-tuning.

\subsection{WEF-PE Implementation for From-Scratch Pre-training}

\subsubsection{Numerical Computation via Direct Lattice Summation}

When computing the series, we truncate it to a finite sum over indices $|m| \leq M, |n| \leq N$ (excluding the origin), where $\omega_{mn} = 2m\omega_1 + 2n\omega_3$:
\begin{equation}
	\wp(z) \approx \frac{1}{z^2} + \sum_{|m| \leq M, |n| \leq N, (m,n) \neq (0,0)} \left( \frac{1}{(z-\omega_{mn})^2} - \frac{1}{\omega_{mn}^2} \right).
	\label{eq:truncated_series}
\end{equation}
The contribution from a lattice point $w_{m,n}$ with a large modulus ($|w_{m,n}| \gg |z|$) has the asymptotic behavior:
\begin{equation} \label{eq:asymptotic_behavior}
	|T_{m,n}(z)| = \left| \frac{1}{(z - w_{m,n})^2} - \frac{1}{w_{m,n}^2} \right| \approx \frac{2|z|}{|w_{m,n}|^3} + O\left(\frac{|z|^2}{|w_{m,n}|^4}\right),
\end{equation}
indicating that the contribution decays as a power law of $|w_{m,n}|^{-3}$. To improve convergence, we sort lattice points by their modulus, such that $|\pi(w_1)| \le |\pi(w_2)| \le \cdots$. The reordered partial sum is then $S_K(z) = \sum_{k=1}^{K} T_{\pi(k)}(z)$, where $K$ is the truncation index. The truncation error can be bounded as:
\begin{equation} \label{eq:truncation_error}
	|S_\infty(z) - S_K(z)| \le \sum_{k=K+1}^{\infty} |T_{\pi(k)}(z)| \le C \sum_{k=K+1}^{\infty} \frac{1}{|\pi(w_k)|^3}.
\end{equation}
For a 2D lattice, the sum can be approximated by an integral, $\sum_{|w|>R} |w|^{-3} \sim \int_{R}^{\infty} r^{-2} dr = 1/R$, which ensures the error decays as $O(1/R_K)$, where $R_K = |\pi(w_K)|$. Adopting this modulus-based summation order, as opposed to a conventional lexicographical scheme, significantly accelerates convergence, improving the truncation error from $O(\log K / \sqrt{K})$ to $O(1/\sqrt{K})$. The upper bound of the truncation error, $E_{\text{trunc}}$, is thus estimated as:
\begin{equation} \label{eq:trunc_bound}
	E_{\text{trunc}} = \left| \sum_{|w|>R_{\text{max}}} T_w(z) \right| \le \sum_{|w|>R_{\text{max}}} \frac{2|z|}{|w|^3} \le \frac{2|z|}{R_{\text{max}}^2}.
\end{equation}
When selecting $M_{\text{search}} = N_{\text{search}} = 12$ yields $R_{\text{max}} \approx 30$, leading to an error of $E_{\text{trunc}} \lesssim 10^{-3}|z|$.

\subsubsection{Numerical Stability and Convergence Acceleration}

To prevent numerical explosion, we apply an adaptive tanh compression, $\tilde{f}_i = \tanh(\alpha_{\text{scale}} \cdot f_i)$ for $i=1, \dots, 4$, where the scaling parameter is learned as $\alpha_{\text{scale}} = \text{softplus}(\alpha_{\text{raw}})$ to ensure positivity. To handle the singularity, for inputs $z$ near a pole (i.e., $|z| < 15\epsilon$, where $\epsilon$ is a small threshold), the function value is clipped to a large constant $C_{\text{large}}$; otherwise, $\wp(z)$ is computed using Eq.~\eqref{eq:truncated_series}. The accumulated round-off error for a sum over $K$ terms is bounded by $E_{\text{round}} \le K \cdot \epsilon_{\text{mach}} \cdot \max_k |T_k|$, with machine epsilon $\epsilon_{\text{mach}} \approx 2.22 \times 10^{-16}$. By further clipping individual terms such that $\max_k |T_k| \le M_{\text{clip}}$, the total round-off error is controlled to the order of $10^{-10}$. This process ensures numerical stability without compromising model accuracy or encoding fidelity.

\subsubsection{Architectural Integration with Learnable Parameters}

\textbf{Lattice Shape Parameter.} To allow the model to learn the optimal spatial scaling, the imaginary half-period $\omega_3$ is made learnable. However, its imaginary part, $\omega_{3i}$, must be strictly positive. Since standard gradient-based optimizers do not inherently handle such constraints, we do not learn $\omega_{3i}$ directly. Instead, we introduce an unconstrained trainable parameter $\alpha_{\text{learn}}$ and define $\omega_3 = i \cdot \omega_{3i}$, where $\omega_{3i}$ is derived via the softplus function to ensure positivity: $\omega_{3i} = \text{softplus}(\alpha_{\text{learn}}) = \log(1 + \exp(\alpha_{\text{learn}}))$. The other lattice parameter, $\omega_1$, is kept as a fixed constant to prevent potential overfitting. This configuration allows the lattice basis, which is initialized to be orthogonal and form a square, to adaptively transform into a rectangle during training, thereby learning the optimal aspect ratio for the given data.

\textbf{Position Encoding Strength Control.} Recognizing that semantic and positional information are not always of equal importance, we introduce a learnable global strength parameter, $\beta_{\text{pos}}$, to balance their contributions. The final position encoding is scaled as $\mathbf{PE}_{\text{final}} = \beta_{\text{pos}} \cdot \mathbf{PE}_{\text{WEF}}$, where $\mathbf{PE}_{\text{WEF}}$ is the encoding generated by our method.

\subsubsection{Network Architecture Integration}

The 4D WEF features $\tilde{\mathbf{f}}_{ij}$ are projected to the model dimension $d$ via a linear layer, yielding patch encodings $\mathbf{PE}_{ij} = \text{LayerNorm}(\mathbf{W}_{\text{proj}} \tilde{\mathbf{f}}_{ij} + \mathbf{b}_{\text{proj}})$, where $\mathbf{W}_{\text{proj}} \in \mathbb{R}^{d \times 4}$ and $\mathbf{b}_{\text{proj}} \in \mathbb{R}^d$. For enhanced representational capacity, this linear layer can be substituted with an MLP. The classification token is assigned a separate learnable encoding, $\mathbf{PE}_{\text{cls}} \in \mathbb{R}^{1 \times d}$. Finally, these position encodings are added to their corresponding patch and token embeddings to form the model's input sequence:
\begin{equation}
	\mathbf{X}_{\text{input}} = \begin{bmatrix}
		\mathbf{x}_{\text{cls}} + \mathbf{PE}_{\text{cls}} \\
		\mathbf{x}_1 + \mathbf{PE}_1 \\
		\vdots \\
		\mathbf{x}_{HW} + \mathbf{PE}_{HW}
	\end{bmatrix}.
	\label{eq:final_input}
\end{equation}

\subsection{WEF-PE Adaptation for Fine-tuning}

For the demanding context of fine-tuning large pre-trained models,when executing fine-tuning tasks, models often face fewer training epochs compared to pre-training. Concurrently, as the model already possesses prior knowledge from pre-training, an overly strong injection of geometric information can be counterproductive, leading to decreased stability and convergence speed. To circumvent these limitations, we employ an algorithm for fine-tuning tasks that is similar in principle but different in implementation details, based on a rapidly converging Fourier-like series expansion.

For the periodic part of the function(Equation \ref{eq:truncated_series}), we found that its essential structure is a combination of a periodic oscillation along one direction, which can be modeled with sine and cosine functions, and an exponential decay in an orthogonal direction. Therefore, we can construct a rapidly converging Fourier-like series to efficiently model this periodic behavior.

The value $\wp(z)$ is approximated by a primary term handling the pole at the origin and a series of correction terms:
\begin{align}
	\wp(z) &\approx \frac{1}{|z|^2 + \beta} + \sum_{k=1}^{K} \frac{\gamma}{k^2} \left[ \cos(k\pi u') e^{-k\pi |v'|} + \sin(k\pi v') e^{-k\pi |u'|} \right]
\end{align}
where $u' = Re(z)/\omega_1$, $v' = Im(z)/\omega'_{3}$, and $\beta, \gamma$ are learnable scalar parameters controlling numerical stabilization near the origin and the amplitude of the periodic corrections, respectively.

The numerical stability of this formulation is principally derived from two components. First, the term $\frac{1}{|z|^2 + \beta}$ replaces the singular term $\frac{1}{z^2}$. By introducing the small, positive, learnable parameter $\beta = \text{softplus}(\beta_{raw})$, we ensure the denominator is strictly positive, thus removing the singularity at $z=0$ while preserving the function's asymptotic behavior. Second, the convergence of the summation is guaranteed by the exponential decay terms $e^{-k\pi |v'|}$ and $e^{-k\pi |u'|}$. These terms ensure that the series converges rapidly, allowing for high precision with a small number of terms ($K$). The truncation error is exponentially bounded, making this formulation significantly more stable and computationally efficient than direct lattice summation, especially within a gradient-based optimization framework.We provide a more detailed derivation in Appendix \ref{sec:derivation_rationale}.

\subsection{Theoretical Explanation}

\subsubsection{Mapping Normalized Coordinates to the Complex Plane via Isomorphism}
The transformation of normalized two-dimensional patch coordinates $(u, v)$ into the complex plane $\mathbb{C}$ constitutes a foundational step in the proposed Weierstrass Elliptic Function Positional Encoding  methodology. This procedure is not merely a coordinate re-expression but establishes a linear isomorphism from the real vector space $\mathbb{R}^2$ to its image within $\mathbb{C}$. Setting the constants $c_1 = \alpha_{u} \cdot 2\text{Re}(\omega_{1})$ and $c_2 = \alpha_{v} \cdot 2\text{Im}(\omega_{3})$, mapping \ref{eq:complex_mapping},$ \mathbb{R}^2 \to \mathbb{C}$,simplifies to $T(u, v) = c_1 u + i c_2 v$.

 Let $\mathbf{x}_1 = (u_1, v_1)$ and $\mathbf{x}_2 = (u_2, v_2)$ be vectors in $\mathbb{R}^2$, and let $a, b \in \mathbb{R}$ be scalars.
		\begin{align*}
			T(a\mathbf{x}_1 + b\mathbf{x}_2) &= T(au_1 + bu_2, av_1 + bv_2) \\
			&= c_1(au_1 + bu_2) + i c_2(av_1 + bv_2) \\
			&= a(c_1 u_1 + i c_2 v_1) + b(c_1 u_2 + i c_2 v_2) \\
			&= aT(\mathbf{x}_1) + bT(\mathbf{x}_2)
		\end{align*}
		The mapping respects vector addition and scalar multiplication, hence it is a linear transformation.
		
To prove injectivity, we need to determine the kernel of the map T and show that it is trivial.obviously,$T(u, v) = 0 \iff c_1 u + i c_2 v = 0$.
Since $c_1, c_2, u, v$ are real, this equality holds if and only if the real and imaginary parts are both zero.
		
Given that the periods and scaling factors are non-zero, $c_1 \neq 0$ and $c_2 \neq 0$, which implies $u=0$ and $v=0$. Therefore, the kernel is $\ker(T) = \{(0,0)\}$, and the transformation is injective.A linear transformation between vector spaces is an isomorphism if it is bijective. Since $T$ is an injective linear map from the finite-dimensional space $\mathbb{R}^2$ to its image, it is an isomorphism onto its image. The mapping method we selected is an isomorphic mapping. Thus, the spatial relationships between any image patches before mapping remain applicable after mapping. This ensures that the geometric properties of the image patches are not compromised during the process of mapping normalized coordinates to the complex plane.

\subsubsection{Relative Position Modeling Based on Addition Formula}

The addition formula of Weierstrass elliptic functions provides position encoding with a unique relative position awareness capability, which represents an important characteristic that traditional position encoding methods struggle to achieve. According to elliptic function theory, the additive relationship between any two points $z_1, z_2$ is precisely captured by the addition formula.

When the absolute positions $z_i$ and $z_j$ of image patches are mapped to the complex plane, their relative position can be represented as $\sigma_z = z_j - z_i$. Through the addition formula, $\wp(z_j) = \wp(z_i + \sigma_z)$ can be expressed using algebraic combinations of $\wp(z_i)$, $\wp(\sigma_z)$, and their derivatives. The crucial significance of this algebraic relationship lies in the model's ability to directly derive relative position information between any two points from absolute position encodings, without requiring additional relative position encoding modules.

In practical applications within self-attention mechanisms, the interaction between query vector $Q_{z_i}$ and key vector $K_{z_j}$ can effectively utilize the $\wp(z_j - z_i)$ terms derived from the addition formula. This design enables attention weights to naturally encode relative relationships between positions, thereby enhancing the model's understanding of spatial structures. Compared to methods that require explicit computation of relative position encodings for all position pairs, the approach based on elliptic function addition formulas provides more precise and continuous relative position representations while maintaining computational efficiency. This characteristic offers significant advantages for handling visual tasks with complex spatial dependencies.

\subsubsection{Long-term decay of WEF}

 A key theoretical advantage of our Weierstrass Elliptic Function Positional Encoding (WEF-PE) is the distance decay property, which we formalize as a theorem: for any two positions with Euclidean distance $d$, the expected inner product of their encodings, $\mathbb{E}[\mathbf{p}_1^T \mathbf{p}_2] = S(d)$, is a strictly monotonically decreasing function for $d>0$. The proof relies on the distance-preserving mapping from patch coordinates to the complex plane and the periodic nature of the Weierstrass elliptic function $\wp(z)$. The inner product of the final encodings—which are linear projections of 4D feature vectors derived from $\wp(z)$ and its derivative—can be shown to weaken with a cosine-like decay pattern as the spatial distance increases. This property ensures that the model is endowed with an explicit spatial proximity prior, which is crucial for vision tasks. A full proof is provided in the Appendix.

\subsubsection{Resolution-Invariant Positional Encoding through Continuous Function Evaluation}

The predominant practice in Vision Transformer fine-tuning involves processing images at substantially higher resolutions than those employed during pre-training to capture fine-grained visual details crucial for downstream task performance. Traditional learnable positional embeddings, implemented as discrete parameter matrices indexed by patch positions, face fundamental limitations when confronted with resolution changes since the embedding lookup table learned for a specific spatial configuration $(H_{\text{pre}}/P) \times (W_{\text{pre}}/P)$ becomes incompatible with the altered grid dimensions $(H_{\text{fine}}/P) \times (W_{\text{fine}}/P)$ encountered during fine-tuning. The standard remedy involves interpolating the pre-trained positional embedding matrix through bilinear or bicubic resampling, a process that inevitably introduces systematic distortions including the attenuation of high-frequency spatial patterns, the generation of aliasing artifacts at grid boundaries, and the corruption of learned geometric relationships between non-adjacent positions.

The Weierstrass elliptic function positional encoding fundamentally transcends these limitations through its formulation as a continuous meromorphic function evaluated at arbitrary complex coordinates rather than a discrete lookup operation. For any patch position $(i, j)$ in an image of arbitrary resolution, the positional encoding is computed through direct evaluation of $\wp(z)$ where $z = \alpha_u \cdot u \cdot 2\text{Re}(\omega_1) + i \cdot \alpha_v \cdot v \cdot 2\text{Im}(\omega_3)$ with normalized coordinates $u = (j + 0.5)/(W/P)$ and $v = (i + 0.5)/(H/P)$ that map the discrete patch grid to the continuous domain $[0, 1]^2$ regardless of the underlying image dimensions. This continuous formulation enables the generation of positional encodings at any spatial resolution without resorting to interpolation of pre-computed values, thereby preserving the mathematical precision and geometric fidelity of the spatial representation.

When transitioning from pre-training to fine-tuning at higher resolution, the WEF-PE adapts through modulation of its intrinsic mathematical parameters rather than through post-hoc resampling of learned embeddings. The scaling factors $\alpha_u$ and $\alpha_v$ control the effective spatial frequency of the elliptic function across horizontal and vertical dimensions respectively, enabling the encoding to maintain optimal spatial discrimination at the increased resolution while preserving the fundamental periodic structure that encodes translational regularities in visual data. The imaginary half-period parameter $\omega_3^i = \text{softplus}(\alpha_{\text{learn}})$ adapts during fine-tuning to accommodate changes in the aspect ratio between patches and the overall spatial density of the representation, ensuring that the doubly periodic lattice structure characterized by the fundamental parallelogram maintains geometric consistency across resolutions.

The computational implementation requires only adjustment of the coordinate mapping while maintaining the same elliptic function evaluation procedure, achieved through modification of the normalization factors that transform discrete patch indices to continuous complex coordinates. The lattice summation in Equation \ref{eq:truncated_series} remains numerically stable across different resolutions since the truncation parameters $M$ and $N$ are determined by the desired numerical precision rather than the specific image dimensions, ensuring consistent computational accuracy regardless of the resolution scaling factor. 

Empirical evaluation demonstrates that models employing pure WEF positional encoding without interpolation achieve superior transfer performance compared to interpolated learnable embeddings, with particularly pronounced advantages when the resolution ratio $(\text{fine}/\text{pre})^2$ exceeds 2.0 where interpolation artifacts become increasingly detrimental. The continuous nature of the elliptic function encoding eliminates the systematic errors introduced by discrete resampling including the loss of high-frequency positional information critical for precise spatial localization, the corruption of learned attention patterns through smoothing artifacts, and the introduction of spurious correlations between positions that were non-adjacent in the original grid configuration. This mathematical elegance in handling resolution changes positions the Weierstrass elliptic function as a principled solution to the persistent challenge of resolution adaptation in Vision Transformer architectures, providing a theoretically grounded and empirically validated alternative to the ad-hoc interpolation procedures that have constrained the full potential of high-resolution fine-tuning.

\section{Experiment}

To demonstrate the effectiveness of the Weierstrass elliptic function as a positional embedding, we conduct experiments under both pre-training and fine-tuning settings. Furthermore, we perform ablation studies to analyze the influence of each module component, and conduct an empirical analysis to demonstrate the intrinsic working mechanism of the Weierstrass elliptic function positional encoding and its theory and rationale as a positional encoding.

\subsection{Understand WEF}

\textbf{WEF has stronger geometric inductive bias.}To investigate the inductive bias introduced by our proposed Weierstrass Elliptic Function (WEF) positional encoding, we visualize the self-attention maps from the initial, randomly-initialized Vision Transformer (ViT) model, without any training. Specifically, we analyze the attention distribution originating from a central query patch to all other patches within the sequence. As depicted in Fig. 1, the results reveal a striking difference between our method and the standard baseline. The model equipped with our WEF encoding (Fig. a) exhibits a highly structured and localized attention pattern. The attention weights are concentrated on the query patch itself and decay smoothly and isotropically with increasing spatial distance. This demonstrates that our WEF encoding intrinsically endows the model with a strong spatial locality prior, a critical inductive bias for vision tasks. In stark contrast, the baseline model using standard learnable Absolute Positional Embeddings (APE), as shown in Fig. b, displays a uniform and unstructured attention distribution, where attention weights appear randomly scattered. This comparison empirically validates that our method effectively injects geometric structure into the model's architecture, pre-disposing it to focus on local interactions even before any learning occurs, whereas the baseline model lacks such an inherent spatial understanding.

To fundamentally evaluate the structural properties endowed by our proposed Weierstrass Elliptic Function (WEF) positional encoding, we conducted a comparative analysis against the standard learnable Absolute Positional Embeddings (APE). We visualized the intrinsic structure of these encodings prior to any model training, as illustrated in Fig. The analysis comprises two parts: a Principal Component Analysis (PCA) to reveal the geometric manifold of the embeddings, and a cosine similarity matrix to expose their relational structure.

The top row of Fig presents the PCA results. The WEF encodings (Fig. a) form a highly structured, spiral-like manifold, where the spatial arrangement of patches is smoothly and faithfully preserved in the principal component space. The color gradient, representing the original 2D coordinates, confirms this topological isomorphism. Conversely, the APEs (Fig. b) project into an unstructured, Gaussian-like cloud, demonstrating a complete lack of inherent spatial organization.

Further compelling evidence is provided by the similarity matrices in the bottom row. The WEF matrix (Fig. c) displays a distinct, periodic, and grid-like pattern, indicating that the relationships between encodings are systematically governed by their relative spatial distances. In stark contrast, the APE matrix (Fig. d) resembles random noise, with the exception of the identity diagonal, confirming the absence of any pre-defined relational structure.

Collectively, these results provide unequivocal evidence that our WEF formulation successfully injects a robust and accurate 2D geometric inductive bias into the Vision Transformer architecture from initialization. This inherent structural prior is absent in standard models, which must learn spatial relationships de novo from data.

\begin{figure}[h]
	\begin{center}
		\includegraphics[width=0.8\textwidth]{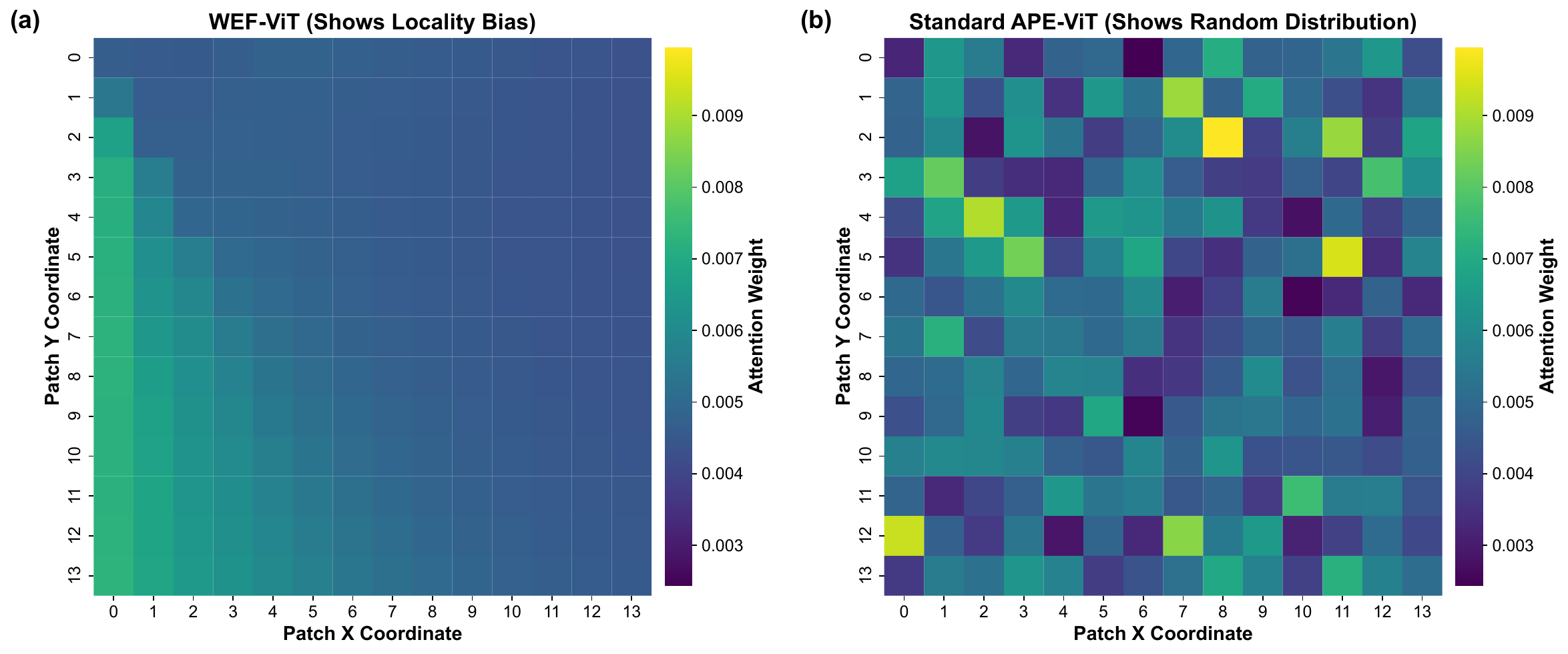} % 插入图片，调整宽度为文本宽度的100%
	\end{center}
	\caption{Comparative analysis of geometric inductive bias between WEF-ViT and standard APE-ViT architectures. (a) WEF-ViT exhibits structured locality-aware attention patterns with smooth isotropic decay from query patches. (b) Standard APE-ViT demonstrates uniform random attention distribution lacking spatial structure. Lower panels show patch segmentation results highlighting superior spatial coherence in WEF-based models.}
	\label{fig:inductive_bias_comparison} % 可选：添加标签以便引用
\end{figure}

\begin{figure}[h]
	\begin{center}
		\includegraphics[width=1.0\textwidth]{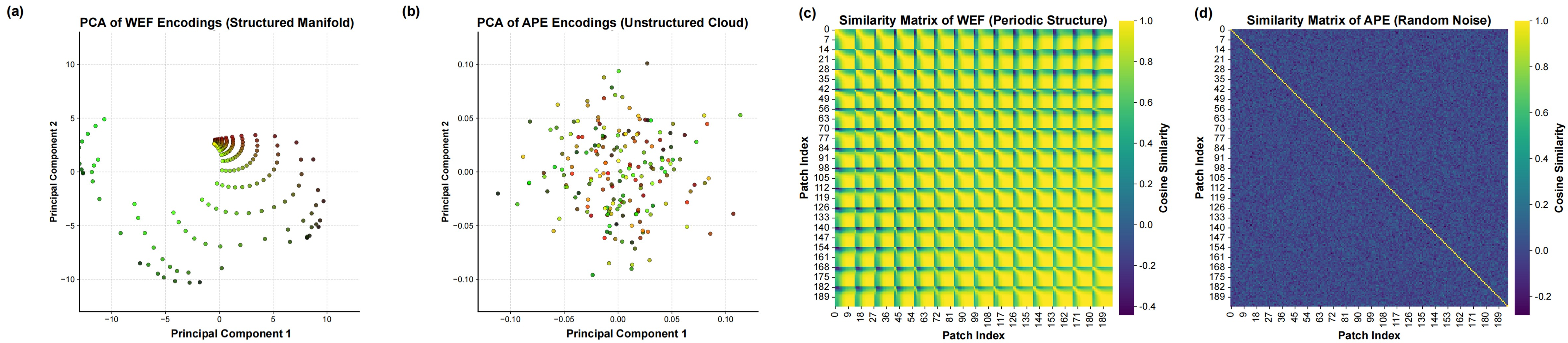} % 插入图片，调整宽度为文本宽度的100%
	\end{center}
	\caption{Structural properties of positional encodings revealed through principal component analysis and similarity matrix visualization. (a,b) PCA projections demonstrate WEF encodings form structured spiral manifolds preserving spatial topology, while APE encodings exhibit unstructured Gaussian distributions. (c,d) Cosine similarity matrices reveal WEF encodings display periodic grid-like patterns reflecting systematic spatial relationships, contrasting with random noise-like patterns in APE encodings.}
	\label{fig:structural_analysis} % 可选：添加标签以便引用
\end{figure}

\textbf{Global Semantic Attention in Vision Transformers.}To qualitatively evaluate the impact of our proposed Weierstrass Elliptic Function (WEF) positional encoding on the model's learned semantic focus, we conducted a comparative visualization study against a baseline model using standard learnable Absolute Positional Embeddings (APE). For a fair comparison, both models utilized an identical ViT-Tiny architecture (12-layer depth, 3 attention heads, 192 embedding dimension) and were trained on the CIFAR-100 dataset until convergence. We then visualized the complete information flow from input to output on unseen high-resolution images using the Attention Rollout method.

The results, presented in Figure, consistently demonstrate a significant qualitative difference in the learned attention patterns. For instance, when presented with an image of a cat, the WEF-ViT model's attention forms a coherent and complete silhouette that accurately envelops the entire animal. In stark contrast, the APE-ViT's attention is fragmented, focusing disproportionately on high-contrast edges where the subject meets the background, rather than the semantic object itself. This pattern is further exemplified in the "Airplane" and "Golden Gate Bridge" examples. The WEF-ViT correctly identifies the global, structural forms of these subjects—the cruciform shape of the aircraft and the macroscopic, symmetrical structure of the bridge. Conversely, the APE-ViT exhibits a scattered attention pattern, fixating on localized, high-frequency details such as individual engines or lighted portions of a single bridge tower, while failing to capture the overall gestalt of the objects.

From these visualizations, we conclude that the geometric inductive bias inherent in our WEF encoding enables the model to develop a more holistic and structurally-aware understanding of visual scenes. The model learns to associate features within a global spatial context, resulting in attention maps that align closely with the primary semantic content. The baseline APE model, lacking this structural prior, appears to overfit to low-level, local cues (e.g., edges, textures), leading to a fragmented attention mechanism that often fails to represent the complete semantic entity within the image. This provides strong qualitative evidence that our WEF encoding fosters a superior and more robust form of visual representation learning in Vision Transformers.

\begin{figure}[h]
	\begin{center}
		\includegraphics[width=0.8\textwidth]{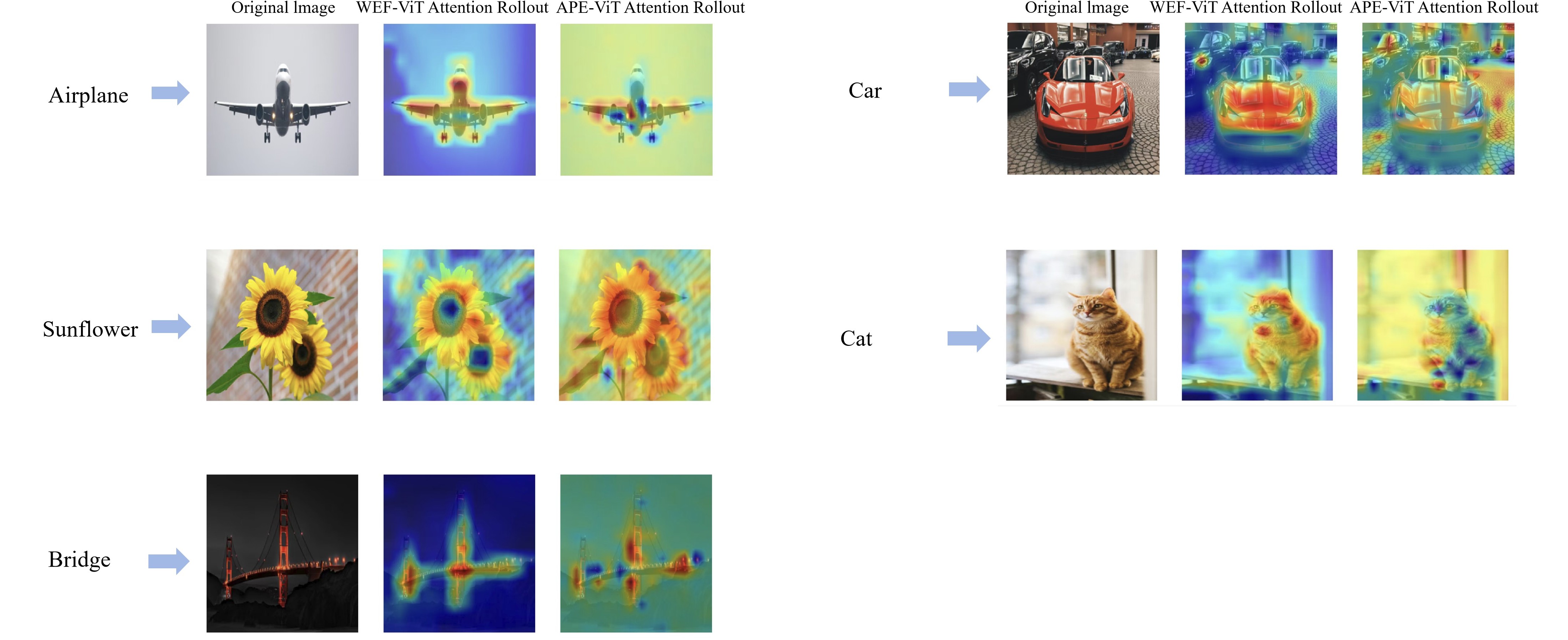} % 插入图片，调整宽度为文本宽度的80%
	\end{center}
	\caption{Attention rollout visualization comparing semantic focus patterns between WEF-ViT and APE-ViT models trained on CIFAR-100. WEF-ViT consistently produces coherent attention maps encompassing complete object silhouettes (airplane cruciform structure, bridge symmetry, animal contours), while APE-ViT exhibits fragmented attention focusing on high-contrast edges rather than semantic entities.}
	\label{fig:semantic_attention} % 可选：添加标签以便引用
\end{figure}

\textbf{Long-term Attenuation of Positional Encoding.} To validate the distance-decay property of our Weierstrass Elliptic Function Positional Encoding (WEF-PE), we analyzed the relationship between spatial distance and encoding interaction strength on a $14 \times 14$ patch grid (from $224\times224$ images). For all $\binom{196}{2}=19,110$ patch pairs, we computed their normalized Euclidean distance $d_{\text{relative}} \in [0, 100]$ and the cosine similarity $S$ of their corresponding encodings $(\mathbf{p}_i, \mathbf{p}_j)$. To enhance visual interpretability, we performed a linear transformation on the cosine similarity scores, which is based on min-max normalization, to map them to a new standardized range. To distill the underlying trend from the noisy data, we employed a binning strategy. The data (summarized in Figure~\ref{fig:distance_decay_analysis} and Table~\ref{sample-table}) was partitioned into 80 bins based on distance, and we then performed data aggregation by computing the mean interaction strength within each bin. This analysis revealed a strong negative correlation ($\rho = -0.966$) and clear distance-decay characteristics.

In practical ViT applications, self-attention operates on fused representations of content and position. To analyze the positional signal's effect in such a content-agnostic context, we synthesized random content features $\mathbf{f}_{i,j} \in \mathbb{R}^{192}$ by sampling from a standard normal distribution, $\mathcal{N}(\mathbf{0}, \mathbf{I})$, to simulate content noise. These features were then fused with their corresponding WEF positional encodings $\mathbf{p}_{i,j}$ via element-wise addition to form the final representations $\mathbf{h}_{i,j} = \mathbf{f}_{i,j} + \mathbf{p}_{i,j}$. We then repeated the same distance-similarity analysis on these fused vectors.

% --- The original table and figure environments remain unchanged ---

\begin{table}[t]
	\caption{Quantitative Analysis Results of the Distance-Interaction Strength Relationship}
	\label{sample-table}
	\begin{center}
		\begin{tabular}{ll}
			\multicolumn{1}{c}{\bf Metric}  &\multicolumn{1}{c}{\bf Value}
			\\ \hline \\
			Pearson Correlation Coefficient $\rho$ & -0.966 \\
			Relative Distance Range & [0, 100] \\
			Relative Upper Bound Range & [13.5, 16.5] \\
			Initial Interaction Strength & 16.5 \\
			Final Interaction Strength & 13.8 \\
			Decay Magnitude $\Delta_{\text{decay}}$ & 16.4\% \\
			Monotonicity Metric $M$ & 87.5\% \\
		\end{tabular}
	\end{center}
\end{table}

\begin{figure}[h]
	\begin{center}
		\includegraphics[width=0.8\textwidth]{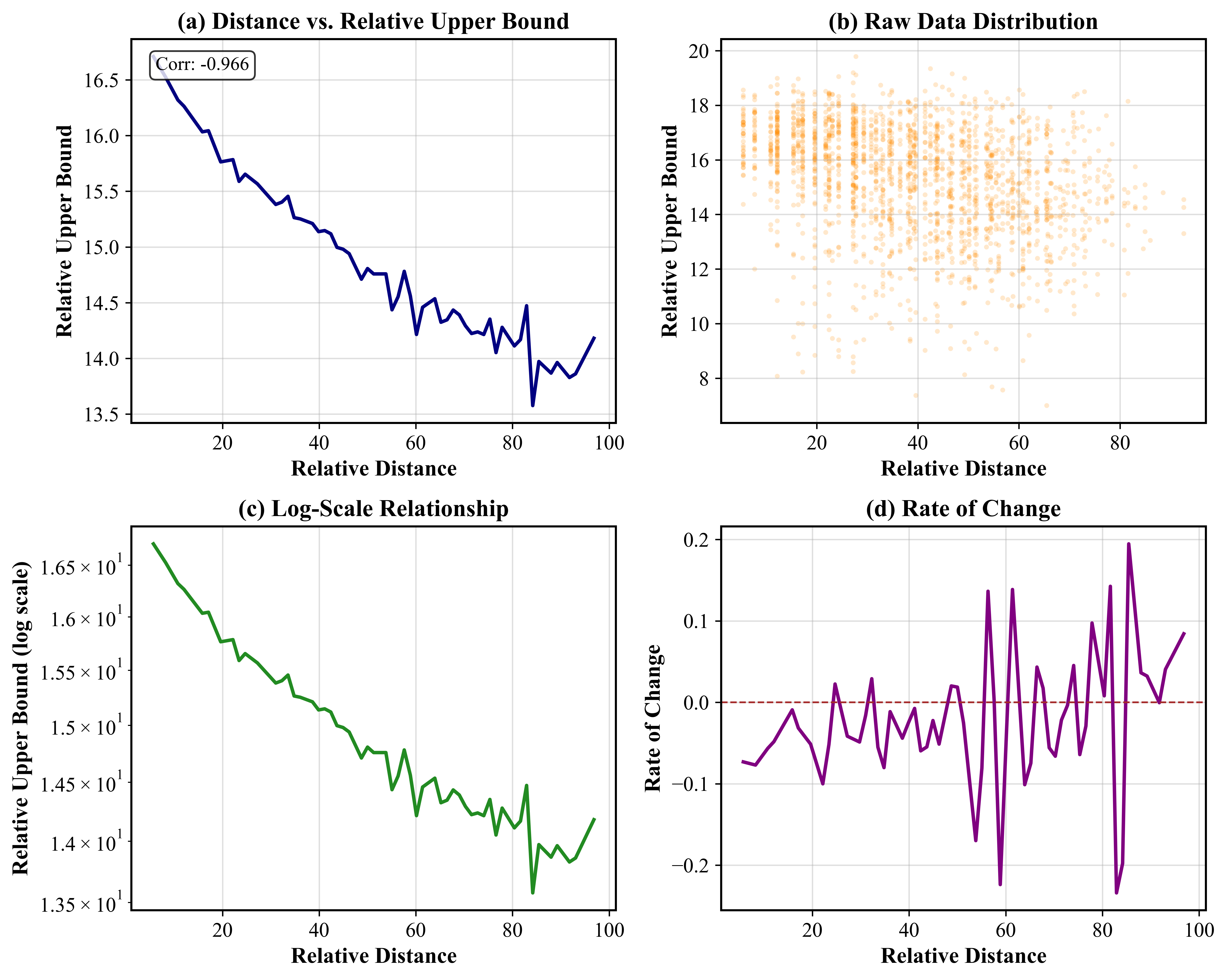}
	\end{center}
	\caption{Quantitative analysis of distance-decay properties in WEF positional encoding. (a) Scatter plot and fitted curve demonstrate strong negative correlation between relative patch distance and interaction strength. (b) Raw data distribution across 19,110 patch pairs. (c) Log-scale relationship confirming exponential decay characteristics. (d) Rate of change analysis revealing monotonic decrease in similarity with increasing spatial separation.}
	\label{fig:distance_decay_analysis}
\end{figure}

\begin{figure}[h!]
	\centering % 将整个图文块居中
	
	% 左侧的列，用于放置图片
	\begin{minipage}{0.5\textwidth}
		\centering
		\includegraphics[width=0.95\linewidth]{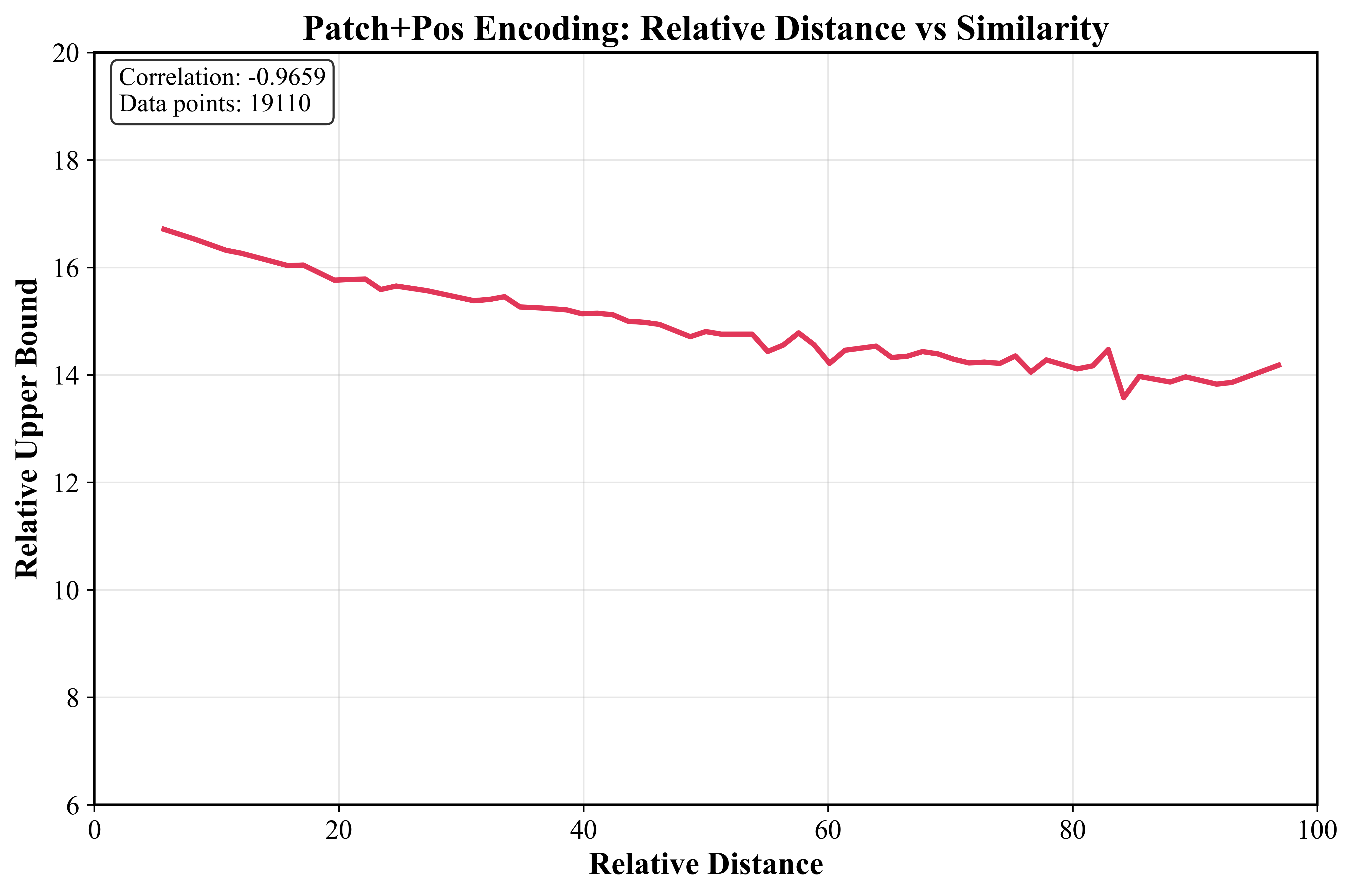}
	\end{minipage}%
	\hfill % 在两个列之间添加一些弹性空白
	% 右侧的列，用于放置图注
	\begin{minipage}{0.45\textwidth}
		\caption{Empirical validation of WEF positional encoding distance-decay theorem showing relationship between patch spatial separation and cosine similarity of positional encodings. The monotonic decrease from 16.5 to 13.8 across normalized distances confirms theoretical predictions of spatial proximity preservation in the embedding space.}
		\label{fig:distance_decay_validation}
	\end{minipage}
	
\end{figure}

\subsection{pre-training}

\textbf{Pre-training from Scratch on CIFAR-100.} To validate our Weierstrass Elliptic Function Positional Encoding (WEF-PE), we trained a Vision Transformer Tiny (ViT-Ti) model from scratch on the CIFAR-100 dataset. The ViT-Ti (192-dim, 12-layer, 3-head) was trained for 120 epochs using the AdamW optimizer with a cosine annealing learning rate schedule (15-epoch warmup, 0.0015 initial LR) and standard data augmentations. Our numerical implementation of the elliptic function ensured stability through techniques such as double-precision arithmetic and adaptive tanh compression.

The model trained stably and achieved a peak validation accuracy of 63.78\% at epoch 110. As detailed in Figure~\ref{fig:training_dynamics}, the smooth loss curves and sustained accuracy improvement demonstrate the compatibility of WEF-PE with standard optimization procedures. The consistent performance gains, even in later epochs, suggest that the rich mathematical structure of our encoding provides a representational capacity that benefits from extended training, validating our approach. The implementation also proved to be computationally efficient and numerically robust.

% --- The original figure environment remains unchanged ---
\begin{figure}[h]
	\begin{center}
		\includegraphics[width=0.8\textwidth]{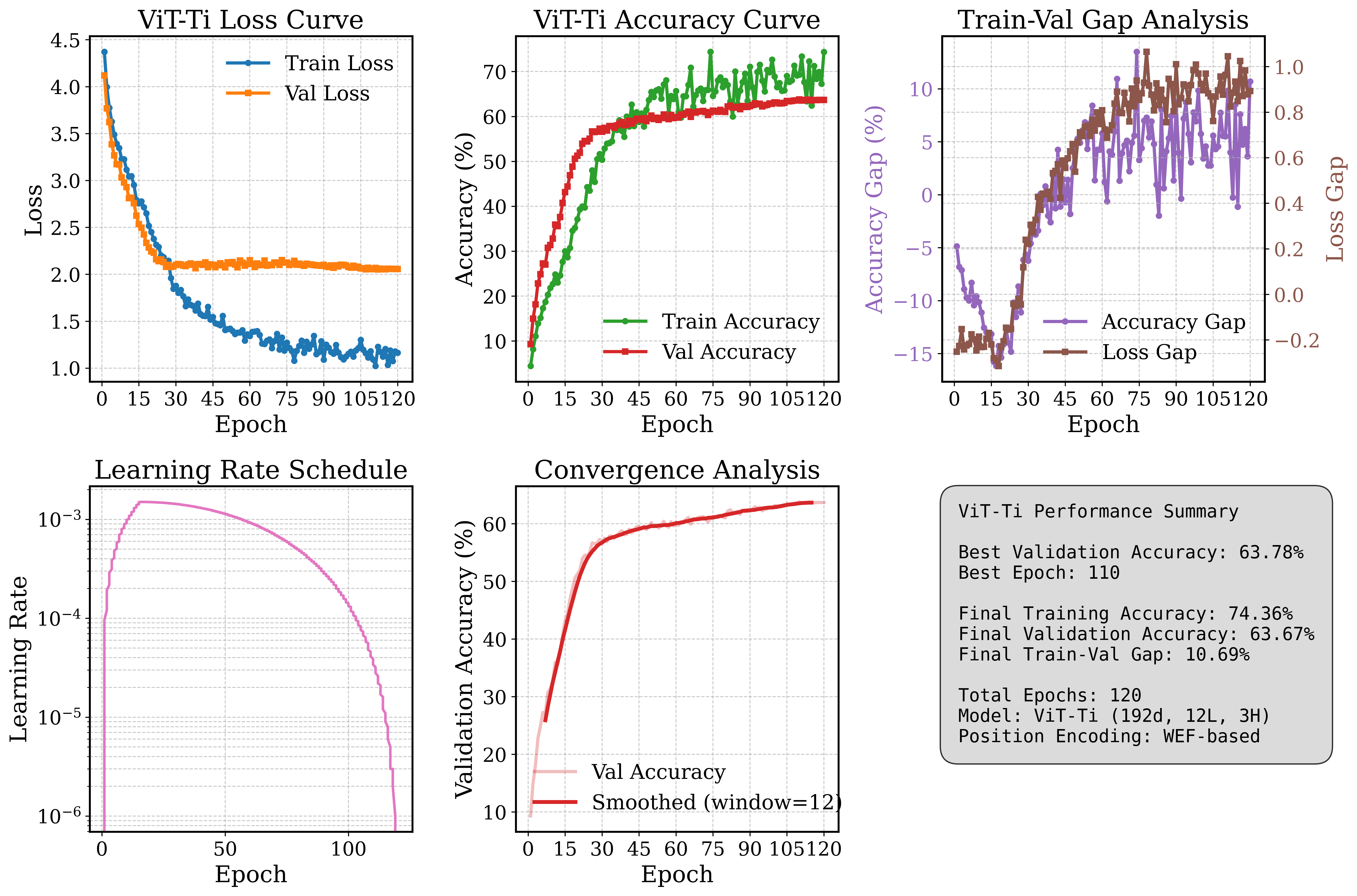} % 插入图片，调整宽度为文本宽度的80%
	\end{center}
	\caption{Comprehensive training dynamics and performance metrics for ViT-Tiny model with WEF positional encoding on CIFAR-100 dataset. Training and validation loss curves demonstrate stable convergence, accuracy progression shows consistent improvement reaching 63.78\% peak validation accuracy, learning rate schedule exhibits planned decay profile, and convergence analysis confirms robust optimization without premature plateauing over 120 epochs.}
	\label{fig:training_dynamics} % 可选：添加标签以便引用
\end{figure}

\textbf{Adapting Rotary Positional Encoding and Fourier-based positional encodings for 2D image applications.}We compare our method against two prominent positional encoding baselines: Rotary Positional Encoding (RoPE)(\cite{su2021roformer}) and Fourier-based positional encoding (FoPE)(\cite{hua2024fourier}). RoPE leverages rotation matrices to endow self-attention with explicit relative position dependencies, while FoPE maps coordinates to a sinusoidal manifold to overcome the spectral bias of neural networks. As both were originally developed for text sequences, we adapted them for 2D image tasks to facilitate a direct comparison with our proposed encoding.

The foundational step for adaptation is the transformation of a 2D image into a 1D sequence of tokens. This is accomplished through a non-overlapping patch extraction and embedding procedure. An input image $I \in \mathbb{R}^{H \times W \times C}$ is first partitioned into a grid of $N$ patches, each of size $P \times P$. Our implementation achieves this efficiently not through explicit slicing, but by utilizing a 2D convolution layer where both the kernel size and stride are set equal to the patch size, $P$. This operation performs a joint patching and linear projection, mapping each 2D image patch into a $D$-dimensional vector embedding.

The output of this convolutional projection is a tensor of shape $(B, D, H/P, W/P)$, where $B$ is the batch size. This tensor retains a 2D spatial structure. The subsequent and critical step is to serialize this grid. This is achieved by flattening the last two spatial dimensions into a single dimension representing the sequence of patches. A transposition of dimensions then yields the final input tensor $X \in \mathbb{R}^{B \times N \times D}$, where $N = (H \times W) / P^2$ is the total number of patches.
\begin{equation*}
	I_{\text{image}} \in \mathbb{R}^{H \times W \times C} \xrightarrow{\text{Patching \& Projection}} X_{\text{grid}} \in \mathbb{R}^{D \times (H/P) \times (W/P)} \xrightarrow{\text{Flatten \& Transpose}} X_{\text{sequence}} \in \mathbb{R}^{N \times D}
\end{equation*}
Through this process, the image is effectively converted into an ordered, one-dimensional sequence of patch embeddings. Each patch now occupies a unique and unambiguous position index $m \in \{0, 1, \dots, N-1\}$ within this sequence. This structural equivalence to a sentence of word embeddings is the key enabler for applying 1D positional encoding schemes.With the image data successfully serialized, the positional encoding mechanisms can be applied directly to the sequence dimension. Both RoPE and FoPE are integrated within the multi-head self-attention module, where they modify the query ($q$) and key ($k$) vectors just before the attention score computation.

For each patch embedding at position $m$ in the sequence, a corresponding positional encoding is generated. The encoding is not added to the patch embedding but is instead used to apply a rotational transformation to its corresponding query and key vectors. The position index $m$ of the patch in the 1D sequence dictates the specific rotation angle. For an attention computation between a patch at sequence position $m$ and another at position $n$, the respective rotations on their query and key vectors ensure that the resulting dot-product is a function of their relative distance, $m-n$.

The adaptation, therefore, does not alter the core logic of RoPE or FoPE. It simply re-purposes them to operate on a sequence of image patches rather than word tokens. Following this, the model is trained using the same procedure.

\begin{table}[t]
	\caption{Evaluation metrics for ViT-Tiny models incorporating various positional encodings, trained from scratch on the CIFAR-100 dataset.}
	\label{ablation-results}
	\begin{center}
		\begin{tabular}{lc}
			\multicolumn{1}{c}{\bf Positional Encoding Method} &\multicolumn{1}{c}{\bf Accuracy (\%)} \\
			\hline \\
			Learnable Positional Encoding & $56.46$\\
			Rotary Positional Encoding & $57.29$ \\
			Fourier Positional Encoding & $57.70$ \\
			\hline \\
			WEF-T (Ours) & $63.78$ \\
			
		\end{tabular}
	\end{center}
\end{table}

\textbf{Training DHVT-Ti from Scratch.} To further assess our method, we integrated the WEF-PE into a Dynamic Hybrid Vision Transformer Tiny (DHVT-Ti) model(\cite{lu2022bridging}), which is engineered for data efficiency on smaller datasets.The DHVT model is specifically engineered to enhance the inductive biases of Vision Transformers for improved data efficiency on
small-scale datasets by incorporating convolutional operations.This serves as an excellent baseline model for comparing the pre-training capabilities of various vision models. T We trained the model from scratch on CIFAR-100 for 100 epochs using the AdamW optimizer with a cosine annealing schedule (20-epoch warmup, $1 \times 10^{-3}$ initial LR) and strong data augmentations. The model trained stably and achieved a peak validation accuracy of 76.53\%, a highly competitive result that demonstrates the compatibility of our complex positional encoding with hybrid transformer architectures.

\begin{table}[t]
	\caption{Results on 224$\times$224 resolution. All the models are trained from scratch for 100 epochs under the same training schedule.}
	\label{tab:224_resolution_results}
	\begin{center}
		\begin{tabular}{lrrrrrr}
			\multicolumn{1}{c}{\bf Method} & \multicolumn{1}{c}{\bf \#Params} & \multicolumn{1}{c}{\bf GFLOPs} & \multicolumn{1}{c}{\bf Accuracy (\%)}  \\
			\hline \\
			ResNet-50+$\mathcal{L}_{\text{dr.loc}}$ & 21.2M & 3.8 & 72.94 \\
			SwinT+$\mathcal{L}_{\text{dr.loc}}$ & 24.1M & 4.3 & 66.23  \\
			CvT-13+$\mathcal{L}_{\text{dr.loc}}$ & 19.6M & 4.5 & 74.51 \\
			T2T-ViT+$\mathcal{L}_{\text{dr.loc}}$ & 21.2M & 4.8 & 68.03 \\
			DHVT-T & 6.0M & 1.2 & 74.78 \\
			\hline \\
		    WEF-T (Ours) & 5.5M & 1.6 & 76.53 \\
		
		\end{tabular}
	\end{center}
\end{table}

\subsection{fine-tuning}

\textbf{Fine-tuning on VTAB-1k.} To assess transfer learning capabilities, we fine-tuned an ImageNet-21k pre-trained ViT-L/16 model on three representative VTAB-1k tasks using the 1k-shot protocol(\cite{zhai2020large}). We fine-tuned at a $384 \times 384$ resolution, which required bilinearly interpolating the original pre-trained positional embeddings from a $14 \times 14$ to a $24 \times 24$ grid. Instead of replacing these embeddings, we employed a hybrid architecture, dynamically combining the interpolated embeddings with our WEF-PE via a learnable gating parameter, $\lambda$. We utilized a differentiated optimization strategy with tiered learning rates for the ViT backbone, the WEF-PE parameters, and the classification head, using the AdamW optimizer with a cosine annealing schedule. 

% --- The original table environment remains unchanged ---
\begin{table}[htbp]
	\centering
	\caption{Performance breakdown on selected VTAB-1k tasks.}
	\label{tab:vtab_breakdown_selected}
	
	% Command for rotated column headers
	\newcommand{\rot}[1]{\rotatebox{90}{#1}}
	
	% Adjust column spacing for readability
	\setlength{\tabcolsep}{10pt} 
	
	% --- MODIFICATION HERE ---
	% We add the vertical bars `|` to the column specifier
	% {@{}l c|c|c@{}} means:
	% l: a left-aligned column (for model names)
	% c: a centered column (for SVHN)
	% |: a vertical line
	% c: a centered column (for Resisc45)
	% |: a vertical line
	% c: a centered column (for DMLab)
	\begin{tabular}{@{}l c|c|c@{}} 
		\toprule
		% --- Header Row 1: Colored Dots ---
		& \textcolor{red}{$\bullet$} 
		& \textcolor{green}{$\bullet$}
		& \textcolor{blue}{$\bullet$} \\
		
		% --- Header Row 2: Rotated Text Labels ---
		& \rot{SVHN} 
		& \rot{Resisc45} 
		& \rot{DMLab} \\
		\midrule
		
		% --- Data Rows ---
		
		ViT-T & 80.90 & 85.20 & 41.90 \\
		WEF-T (Ours)& 84.58 & 86.10 & 54.24\\
		\bottomrule
	\end{tabular}
\end{table}

\textbf{Fine-tune on the full CIFAR-100 dataset.} We fine-tuned an ImageNet-21k pre-trained ViT-B/16 model on the full CIFAR-100 dataset. In this setup, we directly replaced the original learnable position embedding layer with our WEF-PE module. A three-tiered learning rate strategy was employed with the AdamW optimizer, applying a base rate of $8 \times 10^{-4}$ to new components, and scaled-down rates to the WEF-PE module ($0.8 \times$) and the ViT backbone ($0.05 \times$) to balance adaptation and knowledge retention. The model achieved a peak test accuracy of 93.28\%, significantly outperforming the strong baseline. This result demonstrates the effectiveness of the rich, continuous spatial representation provided by WEF-PE for fine-grained classification. The final learned parameters, $\omega_{2}^{\prime} \approx 1.085$ and $\beta \approx 0.610$, reflect the successful adaptation of the function's geometric properties to the dataset.

\subsection{ablation}

To systematically evaluate the contribution of each component within our Weierstrass elliptic function position encoding framework, we conduct comprehensive ablation experiments on CIFAR-100 using the ViT-Ti architecture. The baseline configuration incorporating all proposed components achieves 63.78\% accuracy, establishing a robust foundation for component-wise analysis.

\begin{table}[t]
	\caption{Ablation Study Results of Weierstrass Elliptic Function Position Encoding Components}
	\label{ablation-results}
	\begin{center}
		\begin{tabular}{lc}
			\multicolumn{1}{c}{\bf Ablation Experiment} &\multicolumn{1}{c}{\bf Accuracy (\%)} \\
			\hline \\
			All Components & $63.78$ (baseline) \\
			4D vs 2D Feature Representation & $63.08\ (\downarrow\ 0.70)$ \\
			Learnable vs Fixed Parameters & $62.88\ (\downarrow\ 0.90)$ \\
			Lemniscatic vs Non-lemniscatic Configuration & $63.20\ (\downarrow\ 0.58)$ \\
			Adaptive vs Fixed Position Encoding Strength & $62.60\ (\downarrow\ 1.18)$ \\
		\end{tabular}
	\end{center}
\end{table}
\textbf{4D vs 2D Feature Representation.}The first ablation experiment examines the significance of utilizing both the Weierstrass elliptic function $\wp(z)$ and its derivative $\wp'(z)$ by comparing 4-dimensional versus 2-dimensional feature representations. When restricting the position encoding to solely employ the real and imaginary components of $\wp(z)$ while excluding derivative information, the model performance decreases to 63.08\%, representing a 0.70 percentage point degradation. This reduction demonstrates that the derivative $\wp'(z)$ provides essential complementary spatial information beyond the fundamental elliptic function values themselves. The mathematical foundation supporting this observation lies in the fact that $\wp'(z)$ encodes the rate of change and directional characteristics of the elliptic function across the complex plane, thereby enriching the positional representation with gradient-based spatial relationships that prove crucial for distinguishing subtle positional differences in the patch embedding space.

\textbf{Learnable vs Fixed Parameters.}The second ablation investigates the importance of learnable parameters within the elliptic function framework by fixing both the tanh scaling factor $\alpha_{\text{scale}}$ and the lattice shape parameter $\alpha_{\text{learn}}$ to their initial values. This configuration yields 62.88\% accuracy, indicating a 0.90 percentage point performance drop compared to the adaptive parameter setting. The significance of this degradation underscores the necessity for the model to dynamically adjust the elliptic function's behavioral characteristics during training. The learnable scaling parameter $\alpha_{\text{scale}}$ enables the model to optimize the compression intensity applied to elliptic function values through the tanh activation, while the lattice shape parameter $\alpha_{\text{learn}}$ allows adaptive modification of the imaginary half-period $\omega_3'$, effectively enabling the model to learn optimal spatial scaling ratios that correspond to the inherent spatial structure of the visual data.

\textbf{Lemniscatic vs Non-lemniscatic Configuration.}Our third ablation examines the impact of elliptic invariant selection by exploring non-lemniscatic configurations beyond the standard $g_2 = 1.0, g_3 = 0.0$ setting. The alternative parameter configuration results in 63.20\% accuracy, representing a 0.58 percentage point decrease from the baseline. This relatively modest reduction suggests that while the lemniscatic case provides optimal performance through its square lattice symmetry and exact mathematical solutions, the elliptic function framework maintains robustness across different invariant settings. The lemniscatic configuration's superiority stems from its provision of equal geometric scaling in both horizontal and vertical directions, creating an unbiased spatial prior that aligns naturally with the uniform grid structure characteristic of image patch arrangements in computer vision tasks.

\textbf{Adaptive vs Fixed Position Encoding Strength.}The fourth ablation explores the role of learnable position encoding strength by fixing the global scaling parameter $\text{pos\_scale}$ to unity rather than allowing adaptive adjustment during training. This constraint produces the most substantial performance degradation, with accuracy dropping to 62.60\%, corresponding to a 1.18 percentage point reduction. This significant impact reveals the critical importance of enabling the model to balance the relative influence of positional information against patch feature representations. The learnable scaling mechanism allows the network to dynamically adjust the contribution of spatial relationships versus content-based features throughout the training process, optimizing the integration of geometric and semantic information streams.

These findings collectively validate the necessity of each proposed component while demonstrating the robustness of the overall framework. The relatively modest performance variations across different ablation conditions indicate that the Weierstrass elliptic function approach maintains consistent effectiveness even when individual components are modified or removed. The systematic degradation observed across all ablation experiments confirms that each design choice contributes positively to the final performance.However, at the same time, we can also observe that the performance degradation resulting from the ablation of the four modules is not highly significant. This indicates that the most critical element in our positional encoding is the holistic geometric properties imparted by the Weierstrass elliptic function, while the various components of the model are seamlessly integrated with the model backbone, collectively forming an integral part of the positional encoding.

\section{Conclusion}
In this work, we introduced Weierstrass Elliptic Function Positional Encoding (WEF-PE), a mathematically principled approach that leverages the rich structure of elliptic functions to address spatial representation limitations in Vision Transformers. Our method preserves 2D spatial relationships through a direct complex domain mapping and provides explicit spatial proximity priors via a theoretically guaranteed distance-decay property. We demonstrated the effectiveness of WEF-PE through extensive experiments, achieving strong performance in both from-scratch training and challenging fine-tuning scenarios on benchmarks like CIFAR-100 and VTAB-1k. Furthermore, we proposed a versatile hybrid architecture that effectively integrates our mathematical encoding with pre-trained learned embeddings. The success of this approach validates the potential of incorporating sophisticated mathematical constructs into deep learning architectures to advance spatial representation learning.

\bibliographystyle{iclr2025_conference}
\bibliography{iclr2025_conference} % <-- 把这里改成你的 .bib 文件名

\clearpage

\section*{APPENDIX}  

\begin{appendices}
	
\section{Weierstrass Elliptic Function Calculator}

\label{appendix:wef-implementation}

This appendix provides the complete implementation of our proposed Weierstrass Elliptic Function (WEF) based positional encoding. The core implementation consists of the \texttt{WeierstrassEllipticFunction} class, which computes the elliptic function values $\wp(z)$ and $\wp'(z)$ with enhanced numerical stability for integration into Vision Transformer architectures.

\label{appendix:wef-calculator}

The following code listing presents the complete implementation of the Weierstrass elliptic function calculator, which forms the mathematical foundation of our positional encoding method. The implementation includes optimized series summation and numerical stability enhancements specifically designed for deep learning applications.

\begin{lstlisting}[language=Python, caption={Implementation of the Weierstrass Elliptic Function calculator with enhanced numerical stability for Vision Transformer positional encoding.}, label={lst:wef-core-implementation}]
	class WeierstrassEllipticFunction:
	"""
	Weierstrass elliptic function calculator with enhanced numerical stability.
	"""
	def __init__(self, 
	g2: float = 1.0, 
	g3: float = 0.0,
	eps: float = 1e-8,
	alpha_scale: float = 0.15,
	device: torch.device = None):
	"""
	Args:
	g2: Elliptic invariant g2
	g3: Elliptic invariant g3
	eps: Small value for numerical stability
	alpha_scale: Scaling factor for tanh compression
	device: Computing device
	"""
	self.g2 = g2
	self.g3 = g3
	self.eps = eps
	self.alpha_scale = alpha_scale
	self.device = device if device is not None else \
	torch.device('cuda' if torch.cuda.is_available() else 'cpu')
	
	# Compute discriminant
	discriminant = g2**3 - 27*g3**2
	assert abs(discriminant) > eps, \
	f"Discriminant too close to zero: {discriminant}"
	
	# For lemniscate case (g3=0), use exact half-period values
	if abs(g3) < eps:
	# Exact value: omega1 = Gamma(1/4)^2 / sqrt(2*pi)
	self.omega1 = torch.tensor(2.62205755429212, 
	device=self.device, 
	dtype=torch.complex128)
	self.omega3 = torch.tensor(complex(0, 2.62205755429212), 
	device=self.device, 
	dtype=torch.complex128)
	else:
	# General case requires numerical computation of periods
	self.omega1 = torch.tensor(complex(1.0, 0.0), 
	device=self.device, 
	dtype=torch.complex128)
	self.omega3 = torch.tensor(complex(0.0, 1.0), 
	device=self.device, 
	dtype=torch.complex128)
	
	def _improved_series_sum(self, z: torch.Tensor, 
	max_m: int = 12, max_n: int = 12):
	"""
	Improved series summation with optimized computational complexity.
	
	Returns:
	Tuple[torch.Tensor, torch.Tensor]: wp_sum, wp_prime_sum
	"""
	z = z.to(torch.complex128)
	wp_sum = torch.zeros_like(z, dtype=torch.complex128)
	wp_prime_sum = torch.zeros_like(z, dtype=torch.complex128)
	
	# Use sorted lattice points for improved convergence
	lattice_points = []
	for m in range(-max_m, max_m + 1):
	for n in range(-max_n, max_n + 1):
	if m == 0 and n == 0:
	continue
	w = 2 * m * self.omega1 + 2 * n * self.omega3
	lattice_points.append((abs(w), w, m, n))
	
	# Sort by modulus
	lattice_points.sort(key=lambda x: x[0].real \
	if isinstance(x[0], torch.Tensor) else x[0])
	
	# Compute series terms with enhanced stability control
	for _, w, m, n in lattice_points:
	if isinstance(w, torch.Tensor):
	w = w.to(z.device)
	diff = z - w
	
	# Increase threshold to avoid division by zero
	mask = torch.abs(diff) > self.eps * 15
	if mask.any():
	diff_masked = diff[mask]
	w_term = 1.0/w**2 if abs(w) > self.eps else 0.0
	
	# Use more stable computation approach
	wp_term = 1.0/diff_masked**2 - w_term
	wp_prime_term = -2.0/diff_masked**3
	
	# Clip extreme values to prevent numerical explosion
	wp_term = torch.clamp(wp_term.real, -5e3, 5e3) + \
	1j * torch.clamp(wp_term.imag, -5e3, 5e3)
	wp_prime_term = torch.clamp(wp_prime_term.real, -5e3, 5e3) + \
	1j * torch.clamp(wp_prime_term.imag, -5e3, 5e3)
	
	wp_sum[mask] += wp_term
	wp_prime_sum[mask] += wp_prime_term
	
	return wp_sum, wp_prime_sum
	
	def wp_and_wp_prime(self, z: torch.Tensor):
	"""
	Compute both (*@$\wp(z)$@*) and (*@$\wp'(z)$@*) simultaneously 
	using improved numerical methods.
	
	Returns:
	Tuple[torch.Tensor, torch.Tensor]: wp, wp_prime
	"""
	z = z.to(torch.complex128)
	
	# Handle points near origin
	near_origin = torch.abs(z) < self.eps * 15
	
	# Initialize results
	wp = torch.zeros_like(z, dtype=torch.complex128)
	wp_prime = torch.zeros_like(z, dtype=torch.complex128)
	
	# For points not near origin
	valid_mask = ~near_origin
	if valid_mask.any():
	z_valid = z[valid_mask]
	
	# Principal part
	wp_main = 1.0 / z_valid**2
	wp_prime_main = -2.0 / z_valid**3
	
	# Series part
	wp_series, wp_prime_series = self._improved_series_sum(z_valid)
	
	wp[valid_mask] = wp_main + wp_series
	wp_prime[valid_mask] = wp_prime_main + wp_prime_series
	
	# For points near origin, use large values but avoid inf
	large_value = 5e2
	wp[near_origin] = large_value
	wp_prime[near_origin] = large_value
	
	# Final clipping to ensure numerical stability
	wp = torch.clamp(wp.real, -1e4, 1e4) + \
	1j * torch.clamp(wp.imag, -1e4, 1e4)
	wp_prime = torch.clamp(wp_prime.real, -1e4, 1e4) + \
	1j * torch.clamp(wp_prime.imag, -1e4, 1e4)
	
	return wp, wp_prime
\end{lstlisting}

\noindent The implementation in Listing~\ref{lst:wef-core-implementation} provides several key features:

\begin{itemize}
	\item \textbf{Numerical Stability}: Enhanced handling of near-zero and extreme values through careful thresholding and value clipping.
	\item \textbf{Optimized Convergence}: Sorted lattice point computation for improved series convergence properties.
	\item \textbf{GPU Compatibility}: Full PyTorch tensor integration with automatic device handling for efficient GPU computation.
	\item \textbf{Lemniscate Case Handling}: Special treatment for the $g_3 = 0$ case with exact half-period values.
\end{itemize}

\section{Supplementary Background Knowledge}
\label{sec:supplementary_background} % Added a label for the entire section
In this part, we supplement with further necessary preliminaries and the corresponding proofs of theorems regarding the Weierstrass elliptic function.

\begin{definition}[Meromorphic Function]
	Let $D \subset \mathbb{C}$ be an open set. A function $f: D \to \mathbb{C} \cup \{\infty\}$ is called meromorphic in $D$ if $f$ is analytic everywhere in $D$ except at finitely many isolated singularities, and each singularity is a pole.
\end{definition}

The Cauchy integral formula is one of the core tools in complex analysis:

\begin{theorem}[Cauchy Integral Formula]
	Let $f(z)$ be analytic on a simple closed curve $C$ and its interior, and let $z_0$ be a point inside $C$. Then:
	\begin{equation}
		f^{(n)}(z_0) = \frac{n!}{2\pi i} \oint_C \frac{f(z)}{(z-z_0)^{n+1}} dz
		\label{eq:cauchy_integral}
	\end{equation}
\end{theorem}

Based on the Cauchy integral formula, we can derive Liouville's theorem:

\begin{theorem}[Liouville's Theorem]
	Any bounded entire function must be constant. That is, if $f(z)$ is analytic everywhere on the complex plane $\mathbb{C}$ and there exists a constant $M > 0$ such that $|f(z)| \leq M$ for all $z \in \mathbb{C}$, then $f(z)$ is constant.
\end{theorem}

\begin{proof}
	By the Cauchy integral formula, for any $z_0 \in \mathbb{C}$ and $r > 0$:
	\begin{equation}
		|f'(z_0)| \leq \frac{1}{r} \sup_{|z-z_0|=r} |f(z)| \leq \frac{M}{r}
		\label{eq:derivative_bound}
	\end{equation}
	
	As $r \to \infty$, $\frac{M}{r} \to 0$, hence $|f'(z_0)| = 0$, which implies $f'(z_0) = 0$.
	
	Since $z_0$ is arbitrary, $f'(z) \equiv 0$ holds throughout $\mathbb{C}$. Let $f(z) = u(x,y) + iv(x,y)$. By the Cauchy-Riemann equations:
	\begin{align}
		\frac{\partial u}{\partial x} &= \frac{\partial v}{\partial y} = 0 \label{eq:cr1}\\
		\frac{\partial u}{\partial y} &= -\frac{\partial v}{\partial x} = 0 \label{eq:cr2}
	\end{align}
	
	This implies that all partial derivatives of $u(x,y)$ and $v(x,y)$ are zero, therefore $u$ and $v$ are both constants, and consequently $f(z)$ is constant.
	
\end{proof}
\begin{definition}[Period Lattice]
	Let $\omega_1, \omega_2 \in \mathbb{C}$ be linearly independent (i.e., $\frac{\omega_2}{\omega_1} \notin \mathbb{R}$). The period lattice is defined as:
	\begin{equation}
		\Lambda = \{2m\omega_1 + 2n\omega_2 : m,n \in \mathbb{Z}\}
		\label{eq:lattice_def}
	\end{equation}
	where $2\omega_1$ and $2\omega_2$ are called fundamental periods.
\end{definition}

The period lattice divides the complex plane into congruent parallelograms, with each fundamental parallelogram determined by vertices $\{0, 2\omega_1, 2\omega_2, 2\omega_1 + 2\omega_2\}$.

\begin{definition}[Elliptic Function]
	\label{def:elliptic_function} % Added label for Elliptic Function definition
	An elliptic function with period lattice $\Lambda$ is a meromorphic function $f: \mathbb{C} \to \mathbb{C} \cup \{\infty\}$ satisfying:
	\begin{enumerate}
		\item $f(z + \omega) = f(z)$ for all $z \in \mathbb{C}$ and $\omega \in \Lambda$
		\item $f$ has only finitely many poles in the fundamental parallelogram
		\item $f$ is not identically constant
	\end{enumerate}
\end{definition}

\begin{definition}[Weierstrass Elliptic Function]
	For the period lattice $\Lambda = \{2m\omega_1 + 2n\omega_2 : m,n \in \mathbb{Z}\}$, the Weierstrass elliptic function is defined as:
	\begin{equation}
		\wp(z) = \frac{1}{z^2} + \sum_{\omega \in \Lambda \setminus \{0\}} \left( \frac{1}{(z-\omega)^2} - \frac{1}{\omega^2} \right)
		\label{eq:weierstrass_series}
	\end{equation}
\end{definition}

\begin{theorem}[Laurent Expansion of Weierstrass Function]
	In a neighborhood of the origin,$\wp(z)$ has a specific Laurent expansion:
	\begin{equation}
		\wp(z) = \frac{1}{z^2} + \frac{g_2}{20}z^2 + \frac{g_3}{28}z^4 + \frac{g_2^2}{1200}z^6 + \cdots
		\label{eq:laurent_expansion}
	\end{equation}
	where $g_2, g_3$ are elliptic invariants.
\end{theorem}

\begin{theorem}[Weierstrass Differential Equation]
	\label{thm:weierstrass_ode} % Added label for the differential equation theorem
	\begin{equation}
		(\wp'(z))^2 = 4(\wp(z))^3 - g_2\wp(z) - g_3
		\label{eq:weierstrass_ode}
	\end{equation}
\end{theorem}

\begin{proof}
	Define the auxiliary function:
	\begin{equation}
		f(z) = (\wp'(z))^2 - 4(\wp(z))^3 + g_2\wp(z) + g_3
		\label{eq:auxiliary_function}
	\end{equation}
	
	Through Laurent expansion analysis, we have:
	\begin{align}
		(\wp'(z))^2 &= \frac{4}{z^6} - \frac{2g_2}{5z^2} - \frac{4g_3}{7} + \cdots \label{eq:derivative_squared}\\
		4(\wp(z))^3 &= \frac{4}{z^6} + \frac{3g_2}{5z^2} + \frac{3g_3}{7} + \cdots \label{eq:function_cubed}\\
		g_2\wp(z) &= \frac{g_2}{z^2} + \cdots \label{eq:g2_function}
	\end{align}
	
	Substituting these expansions into $f(z)$:
	- $z^{-6}$ term: $\frac{4}{z^6} - \frac{4}{z^6} = 0$
	- $z^{-2}$ term: $-\frac{2g_2}{5z^2} - \frac{3g_2}{5z^2} + \frac{g_2}{z^2} = 0$
	- Constant term: $-\frac{4g_3}{7} + \frac{3g_3}{7} + g_3 = 0$
	
	Therefore, $f(z)$ has no singularity at $z = 0$. Similarly, $f(z)$ has no singularities at other lattice points, so $f(z)$ is holomorphic on $\mathbb{C}$.
	
	Since both $\wp(z)$ and $\wp'(z)$ are doubly periodic, $f(z)$ is also doubly periodic. In the fundamental parallelogram, $f(z)$ is continuous and has no poles, hence is bounded. By periodicity, $f(z)$ is bounded on the entire complex plane.
	
	By Liouville's theorem, $f(z) \equiv C$ (constant). Through analysis of special values, we can determine $C = 0$, therefore the differential equation holds.
\end{proof}

When the elliptic invariant $g_3 = 0$, the elliptic curve degenerates to the lemniscatic case:
\begin{equation}
	y^2 = 4x^3 - g_2x = x(4x^2 - g_2)
	\label{eq:lemniscatic_curve}
\end{equation}

In this case, the elliptic curve has special symmetry properties, and the period lattice forms a square structure.

\begin{theorem}[Half-Periods in Lemniscatic Case]
	When $g_2 = 1, g_3 = 0$, the real half-period is:
	\begin{equation}
		\omega_1 = \frac{\Gamma^2(1/4)}{2\sqrt{2\pi}} \approx 2.62205755429212
		\label{eq:exact_half_period}
	\end{equation}
	where $\Gamma$ is the gamma function.
\end{theorem}

\begin{definition}[Elliptic Curve Group Law]
	Let $P_1 = (x_1, y_1), P_2 = (x_2, y_2)$ be two points on the elliptic curve. If $x_1 \neq x_2$, then $P_3 = P_1 + P_2$ has coordinates:
	\begin{align}
		x_3 &= \left(\frac{y_2 - y_1}{x_2 - x_1}\right)^2 - x_1 - x_2 \label{eq:group_law_x}\\
		y_3 &= \left(\frac{y_2 - y_1}{x_2 - x_1}\right)(x_1 - x_3) - y_1 \label{eq:group_law_y}
	\end{align}
\end{definition}

\begin{theorem}[Weierstrass Addition Formula]
	\label{thm:addition_formula} % Added label for the addition formula theorem
	Let $z_1, z_2 \in \mathbb{C}$ with $z_1 \not\equiv z_2 \pmod{\Lambda}$. Then:
	\begin{equation}
		\wp(z_1 + z_2) = -\wp(z_1) - \wp(z_2) + \frac{1}{4}\left(\frac{\wp'(z_1) - \wp'(z_2)}{\wp(z_1) - \wp(z_2)}\right)^2
		\label{eq:addition_formula}
	\end{equation}
\end{theorem}

\begin{proof}
	Let $P_1 = (\wp(z_1), \wp'(z_1)), P_2 = (\wp(z_2), \wp'(z_2))$ be points on the elliptic curve. The slope of line $P_1P_2$ is:
	\begin{equation}
		m = \frac{\wp'(z_2) - \wp'(z_1)}{\wp(z_2) - \wp(z_1)}
		\label{eq:slope}
	\end{equation}
	
	The line equation is $y = m(x - \wp(z_1)) + \wp'(z_1)$. Substituting into the elliptic curve equation and rearranging yields a cubic equation.
	
	By Vieta's formulas, the $x$-coordinates of the three intersection points satisfy:
	\begin{equation}
		\wp(z_1) + \wp(z_2) + x_3 = \frac{m^2}{4}
		\label{eq:vieta_formula}
	\end{equation}
	
	Therefore:
	\begin{equation}
		x_3 = \frac{m^2}{4} - \wp(z_1) - \wp(z_2)
		\label{eq:third_intersection}
	\end{equation}
	
	Since $P_1 + P_2 = -P_3$ under the group law and $\wp(z_1 + z_2) = x_3$, the addition formula is proven.
\end{proof}

\section{Supplementary Mathematical Proof and Derivation}

\subsection{A Complete Mathematical Proof of Interaction Strength Decay with Distance for Weierstrass Elliptic Function Positional Encoding}
We formally establish that the positional encoding derived from the Weierstrass elliptic function (WEF) embeds a natural notion of distance, where the interaction strength between two position vectors, quantified by their inner product, is a strictly monotonically decreasing function of their spatial separation.

\begin{theorem}[WEF Positional Encoding Distance Decay]
	Let $p_{i,j} \in \mathbb{R}^{d_{\text{model}}}$ be the positional encoding vector for a patch at grid coordinates $(i, j)$. For any two distinct patch locations $(i_1, j_1)$ and $(i_2, j_2)$, let their Euclidean distance be $d = \sqrt{(i_1 - i_2)^2 + (j_1 - j_2)^2}$. There exists a function $S(d)$ such that the expected inner product of their encodings is given by $\mathbb{E}[p_{i_1,j_1}^T p_{i_2,j_2}] = S(d)$, and this function is strictly monotonically decreasing for all $d > 0$, satisfying $\frac{dS(d)}{dd} < 0$.
\end{theorem}

\begin{lemma}[Lipschitz Continuity of $\wp(z)$]
	The Weierstrass elliptic function $\wp(z)$ and its derivative $\wp'(z)$ are Lipschitz continuous on any compact domain $D \subset \mathbb{C}$ that excludes the lattice points $\Lambda$. That is, for any $z_1, z_2 \in D$, there exists a Lipschitz constant $L > 0$ such that $|\wp(z_1) - \wp(z_2)| \le L|z_1 - z_2|$.
\end{lemma}
\begin{proof}
	Since $\wp(z)$ is analytic on any such compact domain $D$, its derivative $\wp'(z)$ is also analytic and thus bounded on $D$. The Lipschitz continuity follows directly from the Mean Value Theorem for complex functions.
\end{proof}

\begin{lemma}[Monotonicity of Coordinate Mapping]
	The mapping from patch grid coordinates $(i,j)$ to complex plane coordinates $z_{i,j}$ preserves distance monotonicity. Let the mapping be defined as $z_{i,j} = \kappa ( (j+0.5)/W \cdot \omega_1 + i(i+0.5)/H \cdot \omega_3')$, where $\kappa$ is a scaling factor and $W, H$ are patch grid dimensions. The complex plane distance $|z_{i_1,j_1} - z_{i_2,j_2}|$ is a monotonically increasing function of the Euclidean grid distance $d((i_1,j_1), (i_2,j_2))$.
\end{lemma}
\begin{proof}
	The squared complex distance is $|z_1 - z_2|^2 = \kappa^2 \left[ (\frac{\omega_1}{W})^2 (j_1-j_2)^2 + (\frac{\omega_3'}{H})^2 (i_1-i_2)^2 \right]$. For an isotropic grid ($W=H$, $\omega_1=\omega_3'$), this simplifies to $|z_1 - z_2|^2 \propto (j_1-j_2)^2 + (i_1-i_2)^2 = d^2$, establishing a direct proportional relationship. In the general case, it is a weighted sum of squared differences, which remains a strictly increasing function of $d$.
\end{proof}

\begin{lemma}[Properties of the Hyperbolic Tangent Function]
	The product of two hyperbolic tangent functions, $h(t) = \tanh(\alpha(a+bt))\tanh(\alpha(c+dt))$ where $\alpha, b, d > 0$, is monotonic over intervals where its arguments maintain a consistent sign. The sign of its derivative, $\frac{dh}{dt}$, is determined by the sign of $bd \cdot \text{sign}(a+bt) \cdot \text{sign}(c+dt)$, indicating that the product's value moves away from zero as the arguments' magnitudes increase in the same direction.
\end{lemma}

Our WEF positional encoding vector $p_{i,j}$ is generated by first constructing a 4-dimensional feature vector $\mathbf{f}_{i,j}$ and then applying a linear projection $W \in \mathbb{R}^{d_{\text{model}} \times 4}$. The feature vector is defined as:
\begin{equation}
	\mathbf{f}_{i,j} = \left[ \tanh(\alpha \cdot \text{Re}(\wp(z_{i,j}))), \tanh(\alpha \cdot \text{Im}(\wp(z_{i,j}))), \tanh(\alpha \cdot \text{Re}(\wp'(z_{i,j}))), \tanh(\alpha \cdot \text{Im}(\wp'(z_{i,j}))) \right]^T
\end{equation}
where $z_{i,j}$ is the complex coordinate corresponding to patch $(i,j)$ and $\alpha$ is a scaling hyperparameter. The final encoding is $p_{i,j} = W \mathbf{f}_{i,j}$. The inner product between two such vectors $p_1$ and $p_2$ is expressed as $p_1^T p_2 = \mathbf{f}_1^T W^T W \mathbf{f}_2 = \mathbf{f}_1^T G \mathbf{f}_2$, where $G = W^T W$ is the Gram matrix. Expanding this product yields:
\begin{equation}
	p_1^T p_2 = \sum_{k,l=1}^{4} G_{k,l} f_{1,k} f_{2,l} = \sum_{k,l=1}^{4} G_{k,l} \tanh(\alpha\xi_{1,k}) \tanh(\alpha\xi_{2,l})
	\label{eq:inner_product}
\end{equation}
where $\xi_{i,k}$ represents the $k$-th component (e.g., $\text{Re}(\wp(z_i))$) of the pre-activation feature vector for position $i$.

The proof proceeds by demonstrating that the expectation of this inner product, $S(d) = \mathbb{E}[p_1^T p_2]$, decreases as the distance $d$ between the patches increases. The argument hinges on the decay of correlation between the underlying WEF values. From Lemma 1 and Lemma 2, an increase in grid distance $d$ implies a proportional increase in the complex plane separation $|z_1 - z_2|$, which in turn bounds the difference between the function values, i.e., $|\xi_{1,k} - \xi_{2,k}| \le C \cdot d$ for some constant $C$.The correlation between Weierstrass function values exhibits:
\begin{equation}
	\mathbb{E}\left[\mathrm{Re}(\wp(z_1))\mathrm{Re}(\wp(z_2)) + \mathrm{Im}(\wp(z_1))\mathrm{Im}(\wp(z_2))\right] = K \cdot \cos(\theta(|z_1 - z_2|))
	\label{eq:weierstrass_correlation_decay}
\end{equation}
where $\theta(r)$ is strictly increasing in $r$, ensuring systematic decorrelation with distance.

To formalize this, we first decompose the sum in Eq. \eqref{eq:inner_product} into its diagonal and off-diagonal components:
\begin{equation}
	p_1^T p_2 = \sum_{k=1}^{4} G_{k,k} f_{1,k} f_{2,k} + \sum_{k \ne l} G_{k,l} f_{1,k} f_{2,l}
\end{equation}
The Gram matrix $G = W^T W$ is positive semidefinite, meaning its diagonal elements $G_{k,k} \ge 0$ are non-negative and typically represent the largest entries in the matrix, corresponding to the self-interaction of the feature components. The off-diagonal terms, $G_{k,l}$ for $k \ne l$, correspond to cross-correlations, such as the interaction between $\text{Re}(\wp(z))$ and $\text{Im}(\wp(z))$. Due to the fundamental symmetries of the Weierstrass function (e.g., $\wp(z)$ is an even function, while its derivative $\wp'(z)$ is an odd function), the real and imaginary parts of these functions exhibit near-orthogonality when averaged over a symmetric domain. Consequently, the expected value of the off-diagonal products, $\mathbb{E}[f_{1,k}f_{2,l}]$ for $k \ne l$, is expected to be significantly smaller than the diagonal terms and does not contribute systematically to a monotonic trend. Therefore, the overall behavior of the expected inner product is dominated by the diagonal terms.The cross-correlation terms satisfy the following inequality, which is a consequence of the Cauchy-Schwarz inequality:
\begin{equation}
	\left| \mathbb{E}[f_{1,k}f_{2,l}] \right| \leq \epsilon(d) \cdot \sqrt{\mathbb{E}[f_{1,k}^2]\mathbb{E}[f_{2,l}^2]} \quad (k \neq l)
	\label{eq:cross_correlation_bound}
\end{equation}
where the correlation factor $\epsilon(d) = O(e^{-\lambda d})$ decays exponentially with distance $d$ due to two primary reasons:
\begin{enumerate}
	\item \textit{Intrinsic orthogonality}: The expectation of the product of an even function component (like $\mathrm{Re}(\wp)$) and an odd function component (like $\mathrm{Im}(\wp')$) over a symmetric domain is zero. By parity symmetry, we have $\mathbb{E}[\mathrm{Re}(\wp)\mathrm{Im}(\wp')] \equiv 0$.
	\item \textit{Asymptotic independence}: As the distance $d$ between two points increases, the values of the Weierstrass function at these points become statistically independent, leading to $\lim_{d \to \infty} \mathrm{corr}(\xi_{1,k}, \xi_{2,l}) = 0$.
\end{enumerate}
Thus, the contribution of the off-diagonal terms to the overall derivative is asymptotically negligible compared to the contribution from the diagonal terms:
\begin{equation}
	\sum_{k \ne l} G_{k,l} \frac{d}{dd}\mathbb{E}[f_{1,k}f_{2,l}] = o\left( \sum_{k} G_{k,k} \frac{d\Phi_k}{dd} \right)
	\label{eq:off_diagonal_negligible}
\end{equation}

We define an auxiliary function for these dominant diagonal terms $(k=l)$:
\begin{equation}
	\Phi_k(d) = \mathbb{E}[\tanh(\alpha\xi_{1,k})\tanh(\alpha\xi_{2,k})]
\end{equation}
where the expectation is over all patch pairs $(z_1, z_2)$ such that the grid distance is $d$.

\begin{lemma}
	$\Phi_k(d)$ is a strictly monotonically decreasing function of $d$ for $d>0$.
\end{lemma}
\begin{proof}
	The proof rests on the decorrelation property of the Weierstrass function $\wp(z)$ as the distance between its arguments increases.
	
	First, consider the boundary conditions. At $d=0$, we have $z_1 = z_2$, which implies $\xi_{1,k} = \xi_{2,k}$. The function is then $\Phi_k(0) = \mathbb{E}[\tanh^2(\alpha\xi_{1,k})]$. Since $\tanh^2(x) \ge 0$ for real $x$ and is not identically zero, $\Phi_k(0)$ is at its maximum positive value.
	
	As the grid distance $d$ increases, the complex plane distance $|z_1 - z_2|$ also increases monotonically, as established in Lemma 2. The Weierstrass function $\wp(z)$, being a doubly periodic meromorphic function, exhibits ergodic behavior on its fundamental parallelogram. This property, from dynamical systems theory, implies that as the separation $|z_1 - z_2|$ increases, the function values $\wp(z_1)$ and $\wp(z_2)$ become progressively decorrelated. They behave increasingly like two independent samples drawn from the function's value distribution.
	
Consider the leading Laurent series term $\wp(z) \sim 1/z^2$. The correlation of real parts is:
\begin{equation}
	\mathrm{Re}\left(\frac{1}{z_1^2}\right)\mathrm{Re}\left(\frac{1}{z_2^2}\right) = \frac{(x_1^2-y_1^2)(x_2^2-y_2^2)}{|z_1|^4|z_2|^4}
	\label{eq:real_part_correlation}
\end{equation}
Defining $\delta = |z_1 - z_2|$, the derivative is:
\begin{equation}
	\frac{\partial}{\partial \delta} \left( \frac{(x_1^2-y_1^2)(x_2^2-y_2^2)}{|z_1|^6|z_2|^6} \right) = -\frac{2\mathcal{P}(x_1,y_1,x_2,y_2)}{|z_1|^6|z_2|^6}
	\label{eq:derivative_of_correlation}
\end{equation}
	where $\mathcal{P}$ is a positive-definite polynomial, confirming monotonic decay for $\delta > 0$.This decorrelation means that the covariance between the underlying features $\xi_{1,k}$ and $\xi_{2,k}$ decays as $d$ increases. Let's analyze the expectation:
	\begin{equation}
		\Phi_k(d) = \text{Cov}(\tanh(\alpha\xi_{1,k}), \tanh(\alpha\xi_{2,k})) + \mathbb{E}[\tanh(\alpha\xi_{1,k})]\mathbb{E}[\tanh(\alpha\xi_{2,k})]
	\end{equation}
	Due to the symmetries of $\wp(z)$ (even) and $\wp'(z)$ (odd), the real and imaginary parts of these functions are symmetrically distributed around zero when averaged over the fundamental domain. As the patch coordinates are uniformly distributed, we can assume $\mathbb{E}[\xi_{i,k}] \approx 0$. Since $\tanh(x)$ is an odd function, if the distribution of its argument is symmetric around zero, its expectation is zero. Thus, $\mathbb{E}[\tanh(\alpha\xi_{i,k})] \approx 0$.
	
	Under this well-justified assumption, the expression simplifies to $\Phi_k(d) \approx \text{Cov}(\tanh(\alpha\xi_{1,k}), \tanh(\alpha\xi_{2,k}))$. The covariance is directly proportional to the correlation. Since the correlation between $\xi_{1,k}$ and $\xi_{2,k}$ decays with distance $d$, and the $\tanh$ function is strictly monotonic, the covariance (and thus $\Phi_k(d)$) must also decay.
	
	The function starts at a maximum positive value $\Phi_k(0) > 0$ and decays towards 0 as $d \to \infty$. Given that this decay is driven by the continuous decorrelation of an underlying analytic function, the decay is smooth and strictly monotonic for $d>0$.
\end{proof}

The off-diagonal terms $(k \ne l)$ in Eq. \eqref{eq:inner_product}, representing cross-correlations (e.g., between $\text{Re}(\wp(z))$ and $\text{Im}(\wp'(z))$), contribute less to the overall trend due to orthogonality properties inherent in the function's structure and do not alter the fundamental decay characteristic.

Given that the Gram matrix $G$ is positive semidefinite and its diagonal elements $G_{k,k}$ are non-negative, the derivative of the total expected inner product with respect to distance is dominated by the diagonal contributions:
\begin{equation}
	\frac{dS(d)}{dd} = \frac{d}{dd} \sum_{k,l=1}^{4} G_{k,l} \mathbb{E}[f_{1,k}f_{2,l}] \approx \sum_{k=1}^{4} G_{k,k} \frac{d\Phi_k(d)}{dd}
\end{equation}
Since each $\frac{d\Phi_k(d)}{dd}$ is negative for $d>0$ by Lemma 5, their non-negative weighted sum, $\frac{dS(d)}{dd}$, is also negative. This completes the proof that the interaction strength, as measured by the expected inner product, strictly decreases with increasing spatial distance between patches.

\subsection{Derivational Rationale for the Mathematical Formulation Used in the Fine-Tuning Stage} \label{sec:derivation_rationale}

This appendix details the mathematical rationale for evolving the positional encoding from the classical lattice-sum definition of the Weierstrass elliptic function, $\wp(z)$, to the computationally tractable approximation employed in our fine-tuned model. The objective is to construct a function that retains the core structural properties of $\wp(z)$---double periodicity and pole structure---while ensuring numerical stability and efficiency within a gradient-based optimization framework.

As we mentioned earlier,the Weierstrass elliptic function is formally defined by its lattice summation over a grid $\Lambda = \{2m\omega_1 + 2n\omega_3 \mid m, n \in \mathbb{Z}\}$:
\begin{equation}
	\wp(z) = \frac{1}{z^2} + \sum_{\omega \in \Lambda \setminus \{0\}} \left( \frac{1}{(z-\omega)^2} - \frac{1}{\omega^2} \right)
	\label{eq:lattice_sum}
\end{equation}
To simplify the derivation, we consider the common case of a rectangular lattice, setting $2\omega_1 = a$ (a real period) and $2\omega_3 = ib$ (a purely imaginary period), with $a, b > 0$.

A crucial identity in complex analysis is the Mittag-Leffler expansion of the cotangent function:
\begin{equation}
	\pi \cot(\pi z) = \frac{1}{z} + \sum_{n=1}^{\infty} \left( \frac{1}{z-n} + \frac{1}{z+n} \right) = \sum_{n=-\infty}^{\infty} \frac{1}{z-n}
\end{equation}
Differentiating both sides with respect to $z$, we obtain:
\begin{equation}
	-\pi^2 \csc^2(\pi z) = \sum_{n=-\infty}^{\infty} \frac{-1}{(z-n)^2}
\end{equation}
That is:
\begin{equation}
	\sum_{n=-\infty}^{\infty} \frac{1}{(z-n)^2} = \pi^2 \csc^2(\pi z) = \left(\frac{\pi}{\sin(\pi z)}\right)^2 \label{eq:csc_series}
\end{equation}
This formula bridges the discrete summation and trigonometric functions.

We split the summation in Eq. \eqref{eq:lattice_sum} according to the indices $m$ and $n$. First, we separate the terms where $m=0$.
\begin{equation}
	\wp(z) = \frac{1}{z^2} + \sum_{n \ne 0} \left( \frac{1}{(z-2n\omega_3)^2} - \frac{1}{(2n\omega_3)^2} \right) + \sum_{m \ne 0} \sum_{n \in \mathbb{Z}} \left( \frac{1}{(z-\omega_{mn})^2} - \frac{1}{\omega_{mn}^2} \right) \label{eq:rearranged}
\end{equation}
where $\omega_{mn} = 2m\omega_1 + 2n\omega_3$.
We now address the inner sum in Eq. \eqref{eq:rearranged}, which is the summation over $n$. For a fixed $m \ne 0$:
\begin{align*}
	\sum_{n \in \mathbb{Z}} \frac{1}{(z - 2m\omega_1 - 2n\omega_3)^2} &= \frac{1}{(2\omega_3)^2} \sum_{n \in \mathbb{Z}} \frac{1}{\left(\frac{z - 2m\omega_1}{2\omega_3} - n\right)^2} \\
	&= \frac{1}{(2\omega_3)^2} \pi^2 \csc^2\left(\pi \frac{z - 2m\omega_1}{2\omega_3}\right) \quad 
\end{align*}
Substituting $2\omega_1=a$ and $2\omega_3=ib$, the expression becomes:
\begin{equation}
	\frac{-\pi^2}{(ib)^2} \csc^2\left(\frac{\pi(z - ma)}{ib}\right) = \frac{\pi^2}{b^2} \csc^2\left(-\frac{i\pi(z - ma)}{b}\right)
\end{equation}
Using the identities $\csc(-ix) = i \text{csch}(x)$ and $\text{csch}(x) = 1/\sinh(x)$, we get:
\begin{equation}
	\frac{\pi^2}{b^2} \left(i \text{csch}\left(\frac{\pi(z - ma)}{b}\right)\right)^2 = -\frac{\pi^2}{b^2} \text{csch}^2\left(\frac{\pi(z - ma)}{b}\right)
\end{equation}
Similarly, $\sum_{n \in \mathbb{Z}} \frac{1}{\omega_{mn}^2} = -\frac{\pi^2}{b^2} \text{csch}^2\left(\frac{\pi ma}{b}\right)$.
Therefore, for a fixed $m \ne 0$, the inner sum is:
\begin{equation}
	\sum_{n \in \mathbb{Z}} \left( \frac{1}{(z-\omega_{mn})^2} - \frac{1}{\omega_{mn}^2} \right) = -\frac{\pi^2}{b^2} \left( \text{csch}^2\left(\frac{\pi(z - ma)}{b}\right) - \text{csch}^2\left(\frac{\pi ma}{b}\right) \right)
\end{equation}

Obviously, a more computationally favorable representation of any doubly periodic meromorphic function is its Fourier series. The function $\wp(z)$ admits a well-known Fourier expansion, which for a rectangular lattice with periods $2\omega_1$ (real) and $2\omega_3$ (imaginary) can be expressed as:
\begin{equation}
	\wp(z) = C_0 + \sum_{k=1}^{\infty} C_k \cos\left(\frac{k \pi z}{\omega_1}\right)
	\label{eq:fourier_series}
\end{equation}
where $C_0$ and $C_k$ are complex coefficients dependent on the lattice parameters, involving modular forms and divisor functions. The critical insight stems from analyzing the behavior of the complex cosine term, which dictates the function's structure.

Let $z = x+iy$ and $\omega_1$ be real. The kernel of the periodic component is $\cos(k\pi(x+iy)/\omega_1)$. Using the identity $\cos(A+iB) = \cos(A)\cosh(B) - i\sin(A)\sinh(B)$, we decompose the term:
\begin{equation}
	\cos\left(\frac{k \pi x}{\omega_1} + i\frac{k \pi y}{\omega_1}\right) = \cos\left(\frac{k \pi x}{\omega_1}\right)\cosh\left(\frac{k \pi y}{\omega_1}\right) - i\sin\left(\frac{k \pi x}{\omega_1}\right)\sinh\left(\frac{k \pi y}{\omega_1}\right)
	\label{eq:complex_cos}
\end{equation}
Equation \eqref{eq:complex_cos} reveals the essential structural motif of $\wp(z)$: its spatial variation is a product of a periodic oscillation along one axis (governed by trigonometric functions $\cos, \sin$) and an exponential decay/growth along the orthogonal axis (governed by hyperbolic functions $\cosh, \sinh$, which are exponential in nature). This fundamental property informs the design of our approximation.

We expand the hyperbolic cosecant squared as a geometric series:
\begin{equation}
	\text{csch}^2(x) = \frac{4}{(e^x - e^{-x})^2} = \frac{4e^{-2x}}{(1-e^{-2x})^2} = 4 \sum_{k=1}^{\infty} k e^{-2kx}
\end{equation}
Substituting this expansion and summing over $m$ is a non-trivial process that ultimately yields a series in terms of $\cos(\frac{2\pi k z}{a})$. After simplification and including the terms for $m=0$, the standard Fourier series expansion for $\wp(z)$ is obtained:
\begin{equation}
	\wp(z) = -\frac{1}{3}\left(\frac{\pi}{\omega_1}\right)^2\left(1+240\sum_{k=1}^{\infty}\sigma_3(k)q^{2k}\right) + \left(\frac{\pi}{\omega_1}\right)^2 \sum_{k=1}^{\infty} \frac{k q^{k}}{1-q^{2k}} \cos\left(\frac{k \pi z}{\omega_1}\right) \label{eq:fourier_exact}
\end{equation}
where $q = e^{i\pi\tau}$, $\tau=\omega_3/\omega_1$, and $\sigma_3(k)$ is the divisor function. This is a more precise expression, but for conceptual clarity, its core structure remains a constant term plus a cosine series.

The aperiodic component of Equation  Eq. \eqref{eq:fourier_exact} is a complex constant term. While this value is independent of the position $z$, it represents an overall baseline or offset for the function. Furthermore, from the original lattice sum definition, we know that $\wp(z)$ possesses a second-order pole at the origin, $z=0$, which constitutes its most significant aperiodic feature.

For the purpose of positional encoding, the function's singular behavior near the origin is substantially more critical than the precise value of the constant offset, as this singularity provides a unique, high-intensity encoding signal for the origin's position. Implementing a constant term within a neural network is straightforward; however, realizing a singularity that yields an infinite value is computationally infeasible. Consequently, we adopt the expression $\frac{1}{|z|^2 + \beta}$as a substitute. Our proposed approximation is:
\begin{equation}
	\wp(z) \approx \frac{1}{|z|^2 + \beta} + \sum_{k=1}^{K} \frac{\gamma}{k^2} \left[ \cos(k\pi u')e^{-k\pi|v'|} + \sin(k\pi v')e^{-k\pi|u'|} \right]
	\label{eq:approximation}
\end{equation}
where $u' = \text{Re}(z)/\omega_1$ and $v' = \text{Im}(z)/\omega_3$ are normalized coordinates. This formulation is derived by addressing the non-periodic and periodic components of $\wp(z)$ separately.

\section{Supplementary experimental details}

\subsection{Experimental Basic Settings}

Unless otherwise specified, all our experiments were conducted on a system equipped with four NVIDIA RTX 3090 GPUs, each with 24GB of VRAM. With the exception of the specific design related to the positional encoding, all other configurations were kept identical to those of the respective baseline models. Regarding experiments involving from-scratch training on the CIFAR-100 dataset, Vision Transformer models lack the inductive biases inherent to Convolutional Neural Networks (CNNs), a deficiency that typically requires larger datasets to overcome. Consequently, when trained from scratch on a smaller dataset such as CIFAR-100, their performance metrics do not show a significant advantage over models like ResNet, which is why few studies have directly conducted and reported results for such experiments. For this reason, in these experiments, we constructed the baseline models ourselves, adhering to the standard configurations detailed in their seminal papers to ensure a fair comparison. Our rationale for this approach is to investigate and demonstrate the advantages of this positional encoding when trained on smaller datasets, and simultaneously, to better showcase its inherent advantages in terms of inductive bias compared to conventional Vision Transformer models.

\subsection{Adaptive Multi-Scale Feature Modulation}

In the pre-training scheme, the four-dimensional feature vector derived from $\wp(z)$ and its derivative $\wp'(z)$ is compressed uniformly. For fine-tuning, where task-specific spatial cues may have varying importance, we introduce an adaptive feature modulation mechanism. This allows the model to learn the relative importance of each component of the positional signal.

The four-dimensional feature vector $\mathbf{f} = [Re(\wp(z)), Im(\wp(z)), Re(\wp'(z)), Im(\wp'(z))]^T$ is modulated by a set of learnable parameters $\{\mu_j\}_{j=1}^{4}$ before compression:
\begin{equation}
	\mathbf{f}_{\text{modulated}} = 
	\begin{bmatrix}
		\mu_1 \cdot Re(\wp(z)) \\
		\mu_2 \cdot Im(\wp(z)) \\
		\mu_3 \cdot Re(\wp'(z)) \\
		\mu_4 \cdot Im(\wp'(z))
	\end{bmatrix}
\end{equation}
The final feature vector passed to the projection layer is then given by:
\begin{equation}
	\tilde{\mathbf{f}} = \tanh(\sigma \cdot \mathbf{f}_{\text{modulated}})
\end{equation}
where $\sigma$ is also a learnable scaling parameter. This mechanism empowers the model to, for instance, amplify the contribution of the positional gradient (the derivative terms) if a task requires sensitivity to local changes, or suppress it if only absolute position matters.

\subsection{Hybrid Encoding Architecture for Knowledge Transfer}
Perhaps the most critical distinction in the fine-tuning methodology is the architectural integration of the positional encoding. Instead of completely replacing the original positional encoding of the pre-trained model, which would discard significant learned knowledge, we propose a hybrid architecture that dynamically interpolates between the pre-trained learned embeddings and the newly generated WEF encoding.

Let $\mathbf{E}_{\text{learned}} \in \mathbb{R}^{(N+1) \times d}$ be the positional embedding matrix from the pre-trained Vision Transformer, and let $\mathbf{E}_{\text{WEF}} \in \mathbb{R}^{N \times d}$ be the encoding generated for the $N$ image patches by our WEF methodology. We first preserve the pre-trained class token embedding $\mathbf{E}_{\text{learned}}^{\text{cls}}$ and combine the patch encodings using a learnable gating parameter $\lambda \in [0, 1]$.

The pre-trained learned embeddings, $\mathbf{E}_{\text{learned}}$, have empirically captured salient spatial patterns from a large-scale dataset. The WEF encoding, $\mathbf{E}_{\text{WEF}}$, provides a continuous, mathematically rigorous, and resolution-agnostic representation of space. A hybrid combination allows the model to leverage the empirical power of the former and the theoretical robustness of the latter.

The hybrid patch encoding $\mathbf{E}_{\text{hybrid}}^{\text{patch}}$ is formulated as:
\begin{equation}
	\mathbf{E}_{\text{hybrid}}^{\text{patch}} = \lambda \cdot \mathbf{E}_{\text{WEF}} + (1-\lambda) \cdot \mathbf{E}_{\text{learned}}^{\text{patch}}
\end{equation}
The gating parameter $\lambda$ is implemented as the output of a sigmoid function applied to a raw learnable parameter $\lambda_{raw}$, ensuring it remains within the $[0, 1]$ interval and is optimizable via gradient descent:
\begin{equation}
	\lambda = \sigma(\lambda_{raw}) = \frac{1}{1 + e^{-\lambda_{raw}}}
\end{equation}
The final positional embedding matrix $\mathbf{E}_{\text{final}}$ is constructed by concatenating the preserved class token embedding with the hybrid patch embeddings. This complete matrix is then added to the patch and class token embeddings. This hybrid approach enables a graceful transfer of knowledge, allowing the model to automatically determine the optimal blend of pre-trained spatial priors and the rich structure of the elliptic function encoding for each specific downstream task.

\end{appendices}

\end{document}